\documentclass[10pt,journal,compsoc]{IEEEtran}
\usepackage{cite}
\usepackage{graphicx}
\usepackage{url}
\usepackage{amsmath}
\usepackage{amssymb}
\usepackage{amsthm}
\usepackage[caption = false, font=footnotesize]{subfig}
\usepackage{float}
\usepackage{booktabs} 
\usepackage{multirow}
\usepackage{hyperref}

\usepackage[final]{changes}

\usepackage{array}
\usepackage{color}
\usepackage{tabularx}
\usepackage{longtable}
\usepackage{tabu}
\usepackage{algorithm}
\usepackage{algpseudocode}
\newcolumntype{L}[1]{>{\raggedright\let\newline\\\arraybackslash\hspace{0pt}}m{#1}}
\newcolumntype{C}[1]{>{\centering\let\newline\\\arraybackslash\hspace{0pt}}m{#1}}
\newcolumntype{R}[1]{>{\raggedleft\let\newline\\\arraybackslash\hspace{0pt}}m{#1}}
\newcommand{\etal}{\textit{et al}.}
\newcommand{\ie}{\textit{i}.\textit{e}.}
\newcommand{\eg}{\textit{e}.\textit{g}.}

\DeclareMathOperator*{\argmin}{arg\,min}

\def\R{\mathbb{R}}
\newtheorem{lemma}{Lemma}
\makeatletter
\renewcommand*{\ALG@name}{Procedure}
\makeatother
\algdef{SE}[SUBALG]{Indent}{EndIndent}{}{\algorithmicend\ }%
\algtext*{Indent}
\algtext*{EndIndent}

\ifCLASSINFOpdf
\else
\fi
\long\def\comment#1{}

\begin{document}
%
\title{Image Quality Assessment: Unifying Structure and Texture Similarity}
%
%
%

\author{Keyan~Ding,
        Kede~Ma,~\IEEEmembership{Member,~IEEE,}
        Shiqi~Wang,~\IEEEmembership{Member,~IEEE,}
        and~Eero P.~Simoncelli,~\IEEEmembership{Fellow,~IEEE}
\IEEEcompsocitemizethanks{\IEEEcompsocthanksitem Keyan Ding, Kede Ma, and Shiqi Wang are with the Department of Computer Science, City University of Hong Kong, Kowloon, Hong Kong (e-mail:
keyan.ding@my.cityu.edu.hk, kede.ma@cityu.edu.hk, shiqwang@cityu.edu.hk).
\IEEEcompsocthanksitem Eero P. Simoncelli is with the Flatiron Institute of the Simons Foundation, and the Center for Neural Science and the Courant Institute of Mathematical Sciences, New York University, New York, NY 10003, USA (e-mail: eero.simoncelli@nyu.edu).}
}

%
%


\markboth{IEEE Transactions on Pattern Analysis and Machine Intelligence}%
{Shell \MakeLowercase{\textit{et al.}}: Bare Demo of IEEEtran.cls for Journals}

%



\IEEEtitleabstractindextext{
\begin{abstract}
Objective measures of image quality generally operate by comparing pixels of a ``degraded'' image to those of the original. Relative to human observers, these measures are overly sensitive to resampling of texture regions (\eg, replacing one patch of grass with another). Here, we develop the first full-reference image quality model with explicit tolerance to texture resampling. Using a convolutional neural network, we construct an injective and differentiable function that transforms images to multi-scale overcomplete representations. We demonstrate empirically that the spatial averages of the feature maps in this representation capture texture appearance, in that they provide a set of sufficient statistical constraints to synthesize a wide variety of texture patterns. We then describe an image quality method that combines correlations of these spatial averages (``texture similarity'') with correlations of the feature maps (``structure similarity''). The parameters of the proposed measure are jointly optimized to match human ratings of image quality, while minimizing the reported distances between subimages cropped from the same texture images.  Experiments show that the optimized method explains human perceptual scores, both on conventional image quality databases, as well as on texture databases. The measure also offers competitive performance on related tasks such as texture classification and retrieval. Finally, we show that our method is relatively insensitive to geometric transformations (\eg, translation and dilation), without use of any specialized training or data augmentation. Code is available at \url{https://github.com/dingkeyan93/DISTS}.
\end{abstract}

}
\maketitle

\IEEEdisplaynontitleabstractindextext

%
\IEEEpeerreviewmaketitle

\IEEEPARstart{I}{mage} quality assessment (IQA) -- the quantification of human perception of image quality -- is a fundamental problem in both human and computational vision, and is of paramount importance in a variety of real-world applications, such as image restoration, compression, and rendering. For more than 50 years, the mean squared error (MSE) was the standard full-reference method for assessing signal fidelity and quality, and it continues to play a fundamental role in the development of signal and image processing algorithms, despite its poor correlation with human perception~\cite{wang2009mean,girod1993s}. 

A variety of proposed full-reference IQA methods  provide a better account of human perception than MSE~\cite{wang2004image,sheikh2006image,larson:011006,Laparra:17,zhang2018unreasonable,prashnani2018pieapp}, and the Structural Similarity (SSIM) index~\cite{wang2004image} has become a {\it de facto} standard in the field of image processing.
But these methods rely on alignment of the images being compared, and are thus highly sensitive to differences between images of the same texture (e.g., two different cropped regions of the same bed of pebbles). Two samples of the same texture  differ substantially in the precise arrangement of their features, while  appearing nearly the same to a human observer (see Fig.~\ref{fig:texturefailure}).
Since textured surfaces are ubiquitous in photographic images, it is important  to develop objective IQA metrics that are consistent with this aspect of perceptual similarity. Such a metric would allow the development of a new generation of image processing solutions - for example, a compression engine that statistically synthesizes texture regions rather than trying to exactly re-create the pixels of the original image~\cite{Popat1997cluster,balle2011models}.

\begin{figure*}
  \centering
    \subfloat[]{\includegraphics[height=0.2\linewidth]{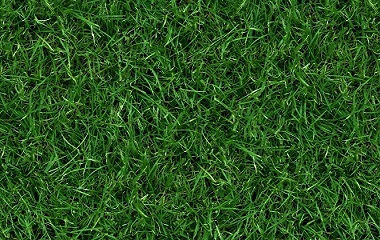}}\hskip.2em
    \subfloat[]{\includegraphics[height=0.2\linewidth]{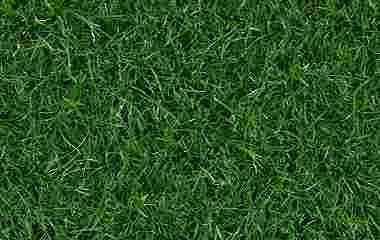}}\hskip.2em 
    \subfloat[]{\includegraphics[height=0.2\linewidth]{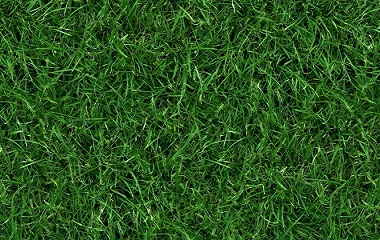}} 
  \caption{\added{Existing full-reference IQA models are overly sensitive to point-by-point deviations between images of the same texture. \textbf{(a)} A grass image and \textbf{(b)} the same image, distorted by JPEG compression. \textbf{(c)} Resampling of the same grass as in (a). Popular IQA measures, including PSNR, SSIM~\cite{wang2004image}, FSIM~\cite{zhang2011fsim}, VIF~\cite{sheikh2006image}, GMSD~\cite{xue2014gradient}, DeepIQA~\cite{bosse2018deep}, PieAPP~\cite{prashnani2018pieapp}, and LPIPS~\cite{zhang2018unreasonable}, predict that image (b) has a better perceived quality than image (c), which is in disagreement with human rating. In contrast, the proposed DISTS model makes the correct prediction. (Zoom in to improve visibility of details).}}
  \label{fig:texturefailure}
\end{figure*}

We present the first full-reference IQA method that is insensitive to resampling of visual textures. Our method is constructed by first nonlinearly transforming images to a multi-scale overcomplete representation, using a variant of the VGG convolutional neural network (CNN)~\cite{Simonyan14c}. We show that the spatial averages of the feature maps provide a compact set of statistical constraints that is sufficient to capture the visual appearance of textures~\cite{portilla2000parametric}.  Specifically, 
we use the test originally proposed by Julesz~\cite{julesz1962visual}, and demonstrate that synthesizing a new image by forcing it to match the channel averages computed from a given texture image results in an image of similar visual appearance. Although the number of statistics in the set is substantially smaller than that of pixels in the image, we find that the result holds for a wide variety of textures, regardless of the initialization, thus revealing the robustness of this model to adversarial examples~\cite{szegedy2013intriguing}.

After transforming the original and corrupted images, we construct our measure by combining two terms over all feature maps: one that compares the spatial averages (and thus, the texture properties) of the two images, and a second that compares the structural details. The final distortion score is computed as a weighted sum of these two terms, with the weights adjusted to match human perception of image quality and invariance to resampled texture patches. The first is achieved by comparing the responses of the model with a database of human image quality ratings. The second is achieved by minimizing the distance between pairs of patches sampled from the same texture images. We show that the resulting Deep Image Structure and Texture Similarity (DISTS) index can be transformed into a proper metric in the mathematical sense. Moreover, DISTS correlates well with human quality judgments in several independent datasets, and achieves a high degree of invariance to texture substitution. We also demonstrate competitive performance of DISTS on tasks of texture classification and retrieval. Last, we show that DISTS is insensitive to mild local and global geometric distortions~\cite{wang2005adaptive,ma2018geometric}, which may be imperceptible to the human visual system (HVS).

\section{Background}
Pioneering work on perceptual full-reference IQA dated back to the 1970s, when Mannos and Sakrison~\cite{mannos1974effects} investigated a class of visual fidelity measures in the context of rate-distortion optimization. 
A number of alternative models were subsequently proposed~\cite{daly1992visible,teo1994perceptual}, each mimicking certain functionalities of the HVS and penalizing the errors between the reference and distorted images ``perceptually''. However, the HVS is a complex and highly nonlinear system~\cite{wandell1995foundations}, and most IQA measures within the error visibility framework rely on strong assumptions and simplifications (\eg, linear or quasi-linear models for early vision characterized by restricted visual stimuli), and exhibit shortcomings regarding the definition of visual quality, quantification of suprathreshold distortions, and generalization to natural images~\cite{wang2006modern}. The SSIM index~\cite{wang2004image} introduced the concept of comparing structure similarity (instead of measuring error visibility), opening the door to a new class of full-reference IQA measures~\cite{wang2005adaptive,wang2005translation,zhang2011fsim,xue2014gradient}. Other design methodologies for knowledge-driven IQA include information-theoretic criterion~\cite{sheikh2006image} and perception-based pooling~\cite{wang2011information}.  Recently, there has been a surge of interest in leveraging advances in large-scale optimization to develop data-driven IQA measures~\cite{ma2018geometric,zhang2018unreasonable,bosse2018deep,prashnani2018pieapp}. However, databases of human quality scores are often insufficiently rich to constrain the large number of model parameters. As a result, these learned methods are at risk of over-fitting~\cite{ma2019group}. 

Nearly all knowledge-driven full-reference IQA models base their quality measurements on point-by-point comparisons between pixels or convolution responses (\eg, wavelets). As such, they are not capable of handling ``visual textures'', which are loosely defined as spatially homogeneous regions with repeated
elements, often subject to some randomization in their
location, size, color, and orientation~\cite{portilla2000parametric}. Different images of the same texture can look nearly the same to the human eye, while differing substantially at the level of pixel intensities. Research on visual texture has a long history, and can be partitioned into four problems: texture classification, texture segmentation, texture synthesis, and shape from texture. At the core of texture analysis is an efficient description (\ie, representation) that matches human perception of visual textures.  
In this paper, we aim to measure perceptual texture similarity, a goal first elucidated and explored in~\cite{clarke2011perceptual,zujovic2013structural}.

The response amplitudes and variances of computational texture features (\eg, Gabor basis functions~\cite{manjunath1996texture}, local binary patterns~\cite{ojala2002multiresolution}) have achieved good performance for texture classification, but are not well correlated with human perceptual ratings of texture similarity~\cite{clarke2011perceptual,zujovic2013structural}. Texture representations that incorporate more sophisticated statistical features, such as correlations of complex wavelet coefficients~\cite{portilla2000parametric}, have shown significantly more power for texture synthesis, suggesting that they may provide a good substrate for similarity measures. In recent years, the use of such statistics extracted from CNN-based representations~\cite{gatys2015texture,Zhang2017Deepten,gao2019perception} has led to even richer texture description.

\section{The DISTS Index} \label{sec:method}

Our goal is to develop a new full-reference IQA model that combines sensitivity to structural distortions (\eg, artifacts due to noise, blur, or compression) with a tolerance of texture resampling (exchanging the content of a texture region with a new sample of the same texture).  
As is common in many IQA methods, we first transform the reference and distorted images to a new representation, using a CNN.  Within this representation, we develop a set of measurements that are sufficient to capture the appearance of a variety of different visual textures.  Finally, we combine these texture parameters with global structural measurements to form an IQA measure.

\subsection{Initial Transformation}
Our model is built on an initial transformation, $f:\mathbb{R}^n\mapsto \mathbb{R}^r$, that maps the reference and distorted images ($x$ and $y$, respectively) to ``perceptual'' representations ($\tilde{x}$ and $\tilde{y}$, respectively). The primary motivation is that perceptual distances are non-uniform in the pixel space~\cite{wang2008maximum,Berardino17c}, and this is the main reason that MSE is inadequate as a perceptual IQA model.  The purpose of function $f$ is to transform the pixel representation to a space that is more perceptually uniform.
Previous IQA methods have used filter banks to capture the frequency-dependence of error visibility~\cite{daly1992visible,larson:011006}.  Others have used transformations that mimic the early visual system \cite{teo1994perceptual,malo2005nonlinear,laparra2010divisive,laparra2016perceptual}.  More recently, deep CNNs have shown surprising power in representing perceptual image distortions~\cite{zhang2018unreasonable,bosse2018deep,prashnani2018pieapp}. In particular, Zhang \etal~\cite{zhang2018unreasonable} have demonstrated that pre-trained deep features from VGG can be used as a substrate for quantifying perceptual quality.

As such, we also chose to base our model on the VGG16 CNN~\cite{Simonyan14c}, pre-trained for object recognition~\cite{krizhevsky2012imagenet} on the ImageNet database~\cite{deng2009imagenet}. 
The VGG transformation is constructed by a feedforward cascade of layers, each including spatial convolution, halfwave rectification, and downsampling.  All operations are continuous and differentiable, both advantageous for an IQA method that is to be used in optimizing image processing systems. We modified the VGG architecture to achieve two additional desired properties. 
First, in order to provide a good substrate for the invariances needed for texture resampling, we wanted the initial transformation to be {\em aliasing-free}.  The ``max pooling'' operation of the original VGG architecture has been shown to introduce visible aliasing artifacts when used to interpolate between images with geodesic sequences~\cite{henaff2015geodesics}. To avoid aliasing when subsampling by a factor of two, the Nyquist theorem requires blurring with a filter whose
cutoff frequency is below $\frac{\pi}{2}$ radians/sample~\cite{oppenheim1999discrete}. Following this principle, we replaced all max pooling layers in VGG with weighted $\ell_2$
pooling~\cite{henaff2015geodesics}:
\begin{align}
P(x) = \sqrt{g*(x\odot x)},
\label{eq:l2pool}
\end{align}
where $\odot$ denotes pointwise product, and the blurring kernel $g(\cdot)$ was implemented
by a Hanning window that approximately enforces the Nyquist criterion \added{with a stride of $2$}. As additional motivation, we note that $\ell_2$ pooling has been used to describe the behavior of complex cells in primary visual cortex~\cite{vintch2015convolutional}, and is also closely related to the complex modulus used in the scattering transform~\cite{bruna2013invariant}.

A second desired property for our transformation is that it should be {\em injective}: distinct inputs should map to distinct outputs.  This is necessary to ensure that the final quality measure is a proper metric (in the mathematical sense) - if the representation of an image is non-unique, then equality of the output representations will not imply equality of the input images. This property has proven useful in perceptual optimization, although it is not present in many recent methods. 
\comment{Earlier IQA measures such as MSE and SSIM used an injective transformation (in fact, the identity mapping), but many more recent methods do not.}
For example, the mapping function in GMSD~\cite{xue2014gradient}
extracts image gradients, discarding local luminance information that is essential to human perception of image quality. Similarly, GTI-CNN~\cite{ma2018geometric}, makes deliberate use of a surjective transformation, in an attempt to achieve invariance to mild geometric transformations, but throws away a substantial amount of structural information that is perceptually important.

Considerable effort has been made in developing invertible CNN-based transformations 
in the context of density modeling~\cite{dinh2014nice,balle17a,kingma2018glow,behrmann19Invertible}. These methods place strict constraints on either network architectures~\cite{dinh2014nice,kingma2018glow} or network parameters~\cite{behrmann19Invertible}, which limit the expressiveness in learning quality-relevant representations. Ma \etal~\cite{ma2018invertibility} proved that under Gaussian-distributed random weights and ReLU nonlinearity, a two-layer CNN is injective provided that it is sufficiently expansive (\ie, the output dimension of each layer should increase by at least a logarithmic factor). Although  mathematically appealing, this result does not constrain parameter settings of CNNs of more than two layers. In addition, a Gaussian-weighted  CNN is less likely to be perceptually relevant~\cite{gatys2015texture,ma2018geometric}. 

Like most CNNs, VGG discards information at each stage of transformation. To ensure an injective mapping, we simply included the input image as an additional feature map (the ``zeroth'' layer of the network).  
The representation then consists of the input image $x$,  concatenated with the convolution responses of five VGG layers (labelled $\text{conv1}\_\text{2}$, $\text{conv2}\_\text{2}$, $\text{conv3}\_\text{3}$, $\text{conv4}\_\text{3}$, and $\text{conv5}\_\text{3}$):
\begin{align}
    f(x) = \{\tilde{x}^{(i)}_j; i = 0,\ldots,m; j = 1,\ldots,n_i\}, \label{eq:vggfx}
\end{align}
where $m=5$ denotes the number of convolution layers chosen to construct $f$, $n_i$ is the number of feature maps in the $i$-th convolution layer, and $\tilde{x}^{(0)} = x$. 
Similarly, we also computed the representation of the distorted image:
\begin{align}
    f(y) = \{\tilde{y}^{(i)}_j; i = 0,\ldots,m; j = 1,\ldots,n_i\}.  \label{eq:vggfy}
\end{align}

\added{We used a na\"{i}ve task --  reference image recovery -- to visually demonstrate the necessity of injective feature transformations. Specifically, given an original image $x$ and an initial image $y_0$, we aim to recover $x$ by numerically optimizing $y^\star =\mathop{\arg\min}_{y} D(x,y)$, 
where $D$ denotes a full-reference IQA measure with a lower score indicating higher predicted quality, and $y^\star$ is the recovered image. For example, if $D$ is the MSE, the (trivial) analytical solution is $y^\star = x$, indicating full recoverability. For the majority of existing IQA models, which are continuous and differentiable, solutions must be sought numerically, using gradient-based iterative solvers.} Fig.~\ref{fig:optim} 
 shows the recovery results of our method from a JPEG-corrupted copy of the original image and a white Gaussian noise image, respectively, in comparison to three state-of-the-art models: GTI-CNN~\cite{ma2018geometric}, GMSD~\cite{xue2014gradient}, and LPIPS~\cite{zhang2018unreasonable}. The first two, which are based on surjective mappings, fail dramatically on this simple task when initialized with purely white Gaussian noise. LPIPS, which is built on VGG but with no enforcement of the injective property, recovers most structures and details, but leaves some visible artifacts in the converged image (Fig. \ref{fig:optim} (j)).
 In contrast, DISTS successfully recovers the reference image from any initialization.

\begin{figure*}
  \centering
  \begin{tabular}{cccccc}
    \multirow{2}{*}[12pt]{
    \subfloat[Reference]{\includegraphics[height=0.15\linewidth]{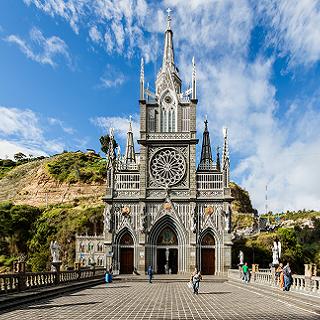}}} &
    
    \subfloat[Initial]{\includegraphics[height=0.15\linewidth]{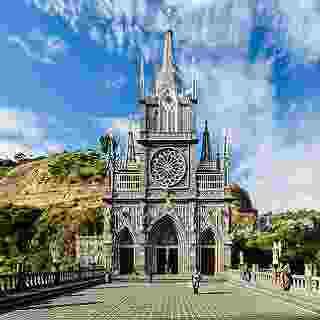}} \hskip.2em
    \subfloat[GTI-CNN~\cite{ma2018geometric}]{\includegraphics[height=0.15\linewidth]{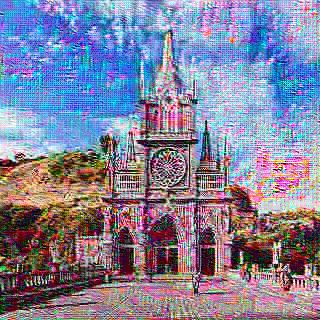}}  \hskip.2em
    \subfloat[GMSD~\cite{xue2014gradient}]{\includegraphics[height=0.15\linewidth]{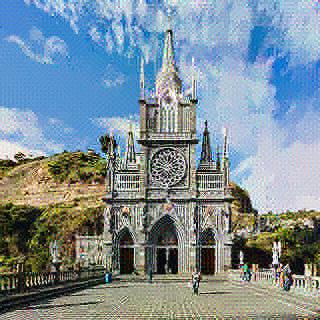}} \hskip.2em
    \subfloat[LPIPS~\cite{zhang2018unreasonable}]{\includegraphics[height=0.15\linewidth]{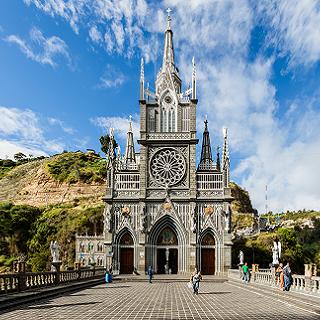}} \hskip.2em
    \subfloat[DISTS (ours)]{\includegraphics[height=0.15\linewidth]{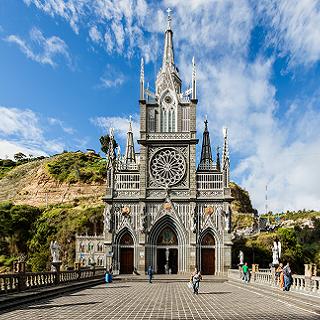}} \\ &
    
    \subfloat[Initial]{\includegraphics[height=0.15\linewidth]{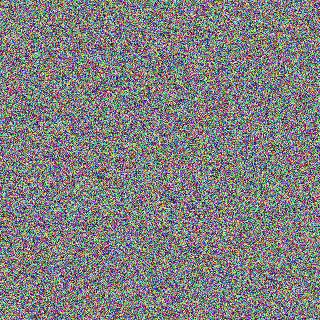}} \hskip.2em
    \subfloat[GTI-CNN~\cite{ma2018geometric}]{\includegraphics[height=0.15\linewidth]{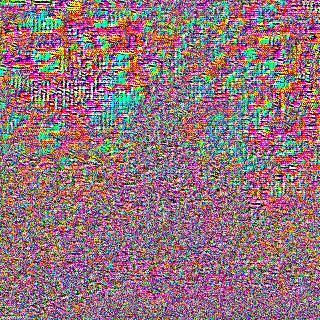}}  \hskip.2em
    \subfloat[GMSD~\cite{xue2014gradient}]{\includegraphics[height=0.15\linewidth]{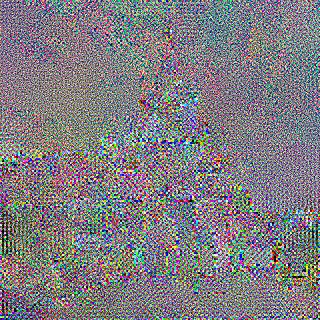}} \hskip.2em
    \subfloat[LPIPS~\cite{zhang2018unreasonable}]{\includegraphics[height=0.15\linewidth]{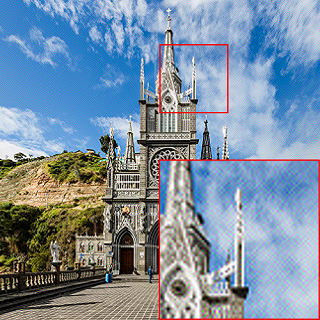}} \hskip.2em
    \subfloat[DISTS (ours)]{\includegraphics[height=0.15\linewidth]{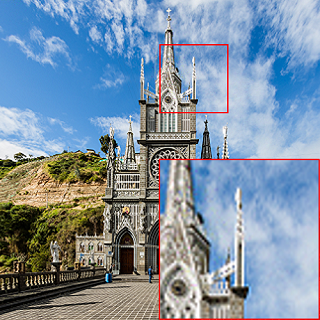}} 
        
    \end{tabular}
  \caption{\added{Recovery of a reference image by optimization of IQA measures. Recovery is implemented by solving $y^\star =\argmin_y D(x,y)$ with gradient descent, where $D$ is an IQA distortion measure and $x$ is a given reference image. \textbf{(a)} Reference image. \textbf{(b)} Corrupted initial image $y_0$, obtained by compressing the reference image using JPEG at a low bitrate. 
  \textbf{(c)-(f)} Images recovered from (b) by optimizing different metrics (as indicated). 
  \comment{by 
  GTI-CNN~\cite{ma2018geometric}, GMSD~\cite{xue2014gradient}, LPIPS~\cite{zhang2018unreasonable}, and the proposed DISTS, respectively.}
  \textbf{(g)} Corrupted initial image, obtained by adding white Gaussian noise. 
  \textbf{(h)-(k)} Images recovered from (g) by optimizing indicated metrics. 
  In all cases, the optimization converges, yielding a distortion score substantially lower than that of the initial. }}  
  \label{fig:optim}
\end{figure*}

\subsection{Texture Representation}
The visual appearance of textures is often characterized in terms of sets of local statistics~\cite{julesz1962visual} that are presumably measured by the HVS. Models consisting of various sets of features~\cite{heeger1995pyramid,zhu1998filters,portilla2000parametric,gatys2015texture} have been 
tested using synthesis: one generates an image with statistics that match those of a texture photograph.  If the set of statistical measurements is a complete description of the appearance of the texture, then the synthesized image should be perceptually indistinguishable from the original~\cite{julesz1962visual}, at least based on preattentive judgments~\cite{malik1990preattentive}. 

Portilla and Simoncelli \cite{portilla2000parametric} found that the local correlations (and other pairwise statistics) of complex wavelet responses were sufficient to capture the visual appearance of a wide variety of textures, while at the same time being of low enough dimensionality ($\sim 700$ dimensions). 
Gatys \etal \cite{gatys2015texture} used correlations across channels of many layers in a VGG network, and were able to synthesize consistently better textures, albeit with a much larger set of statistics ($\sim306$K parameters). Since 
the number of statistics is typically larger than that of pixels in the input image, it is likely that this image was unique in matching these statistics.  In this case, diversity in the synthesis results  reflects local optima of the optimization procedure, rather than the entropy of the implicitly represented probability distribution. Ustyuzhaninov~\etal~\cite{ustyuzhaninov2017does} provided more direct evidence of this hypothesis: If the number of the statistical measurements is sufficiently large (on the order of millions), a single-layer CNN with random filters can always produce textures that are visually indiscernible to the human eye.
Subsequent results suggest that a reduced set of statistics, containing only the mean and variance of CNN channels, is sufficient for texture classification or style transfer~\cite{gatys2016image,dumoulin2016learned,Li2017Demystifying}.

In our experiments, we found that an even more reduced set, containing only the spatial means of the feature maps (a total of $1,475$ statistics), provides an effective parametric model for visual textures. Specifically, we used this model to synthesize textures~\cite{portilla2000parametric} by solving
 \begin{align}
     y^\star =\argmin_y D(x,y)=\argmin_{y} \sum_{i,j}\left(\mu_{\tilde{x}_j}^{(i)}-\mu_{\tilde{y}_j}^{(i)}\right)^2,
     \label{eq:texture_syn}
 \end{align}
where $x$ is the target texture image, and $y^\star$ is the synthesized texture image, obtained by gradient descent optimization from a random initialization. $\mu_{\tilde{x}_j}^{(i)}$ and $\mu_{\tilde{y}_j}^{(i)}$ are the spatial averages of channels $\tilde{x}^{(i)}_j$ and $\tilde{y}^{(i)}_j$, respectively. 
Fig.~\ref{fig:g_mean2} shows the synthesis results of our texture model using statistical constraints from individual and combined convolution layers of the pre-trained VGG. Similar to observations in Gatys~\etal~\cite{gatys2015texture}, we found that measurements from early layers appear to capture basic intensity and color information, and those from later layers summarize the shape and structure information. When matching statistics up to layer $\text{conv5}\_\text{3}$, the synthesized texture appears visually similar to the reference.

Fig.~\ref{fig:synthesis} shows three synthesis results of our $1475$-parameter texture model in comparison with the $710$-parameter texture model of Portilla \&  Simoncelli~\cite{portilla2000parametric} and  the $\sim 306$k-parameter model of Gatys~\etal~\cite{gatys2015texture}. As one might expect, the visual quality of samples synthesized by our model lies between the other two. 

\begin{figure*}[t]
  \centering
  \begin{tabular}{cccccc}
    \multirow{2}{*}[30pt]{\subfloat[]{\includegraphics[height=0.14\linewidth]{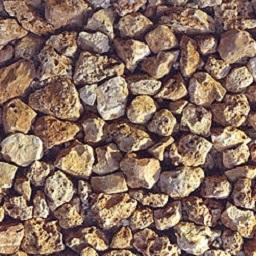}}} & 
    \subfloat[]{\includegraphics[height=0.14\linewidth]{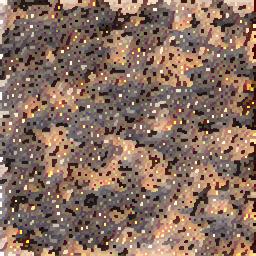}}  \hskip.2em
    \subfloat[]{\includegraphics[height=0.14\linewidth]{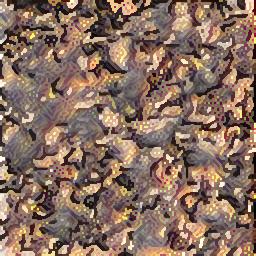}}  \hskip.2em
    \subfloat[]{\includegraphics[height=0.14\linewidth]{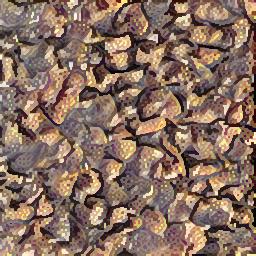}}  \hskip.2em
    \subfloat[]{\includegraphics[height=0.14\linewidth]{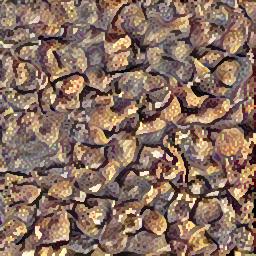}}  \hskip.2em
    \subfloat[]{\includegraphics[height=0.14\linewidth]{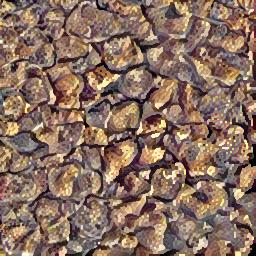}}  \\
    &
    \subfloat[]{\includegraphics[height=0.14\linewidth]{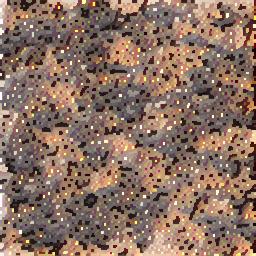}}  \hskip.2em
    \subfloat[]{\includegraphics[height=0.14\linewidth]{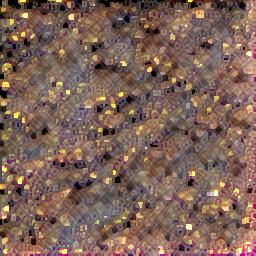}}  \hskip.2em
    \subfloat[]{\includegraphics[height=0.14\linewidth]{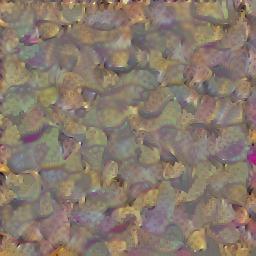}}  \hskip.2em
    \subfloat[]{\includegraphics[height=0.14\linewidth]{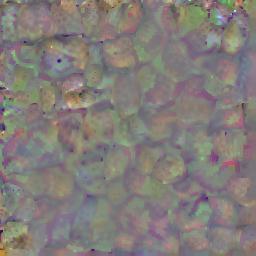}}  \hskip.2em
    \subfloat[]{\includegraphics[height=0.14\linewidth]{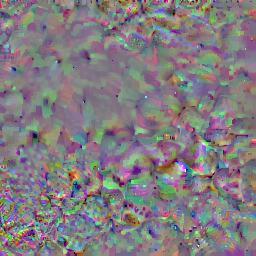}} 
    \end{tabular}
  \caption{Images synthesized to match the mean values of channels  up to a given layer (top) or from individual layers (bottom) of the pre-trained VGG network. 
  \textbf{(a)} Reference texture. 
  \textbf{(b)} Up to $\text{conv1}\_\text{2}$. 
  \textbf{(c)} Up to $\text{conv2}\_\text{2}$. 
  \textbf{(d)} Up to $\text{conv3}\_\text{3}$. 
  \textbf{(e)} Up to $\text{conv4}\_\text{3}$. 
  \textbf{(f)} Up to $\text{conv5}\_\text{3}$.  
  \textbf{(g)} Only $\text{conv1}\_\text{2}$. 
  \textbf{(h)} Only $\text{conv2}\_\text{2}$. 
  \textbf{(i)} Only $\text{conv3}\_\text{3}$. 
  \textbf{(j)} Only $\text{conv4}\_\text{3}$. 
  \textbf{(k)} Only $\text{conv5}\_\text{3}$.} 
  \label{fig:g_mean2}
\end{figure*}

 \begin{figure}
  \centering
    \subfloat{\includegraphics[height=0.24\linewidth]{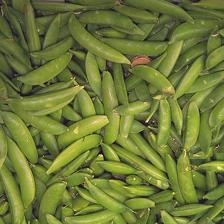}}\hskip.2em
    \subfloat{\includegraphics[height=0.24\linewidth]{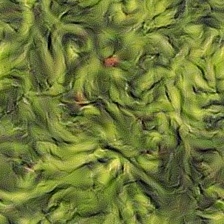}}\hskip.2em
    \subfloat{\includegraphics[height=0.24\linewidth]{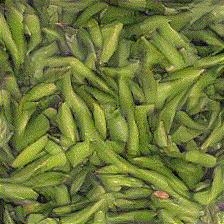}}\hskip.2em
    \subfloat{\includegraphics[height=0.24\linewidth]{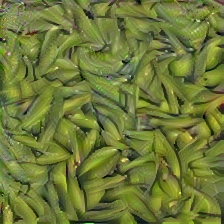}} \\
    \vspace{-.2cm}   
    \subfloat{\includegraphics[height=0.24\linewidth]{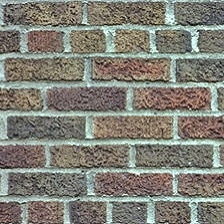}}\hskip.2em
    \subfloat{\includegraphics[height=0.24\linewidth]{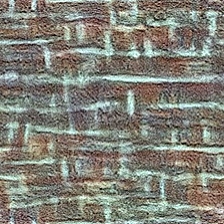}}\hskip.2em
    \subfloat{\includegraphics[height=0.24\linewidth]{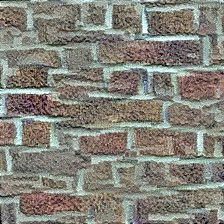}}\hskip.2em
    \subfloat{\includegraphics[height=0.24\linewidth]{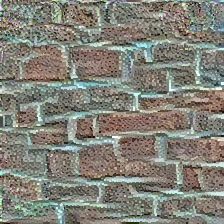}} \\
    \addtocounter{subfigure}{-8}
    \vspace{-.2cm}   
    \subfloat[]{\includegraphics[height=0.24\linewidth]{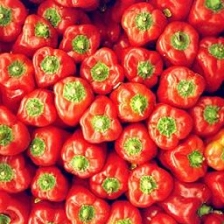}}\hskip.2em
    \subfloat[]{\includegraphics[height=0.24\linewidth]{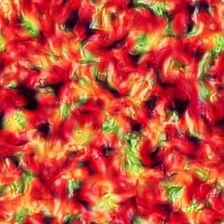}}\hskip.2em
    \subfloat[]{\includegraphics[height=0.24\linewidth]{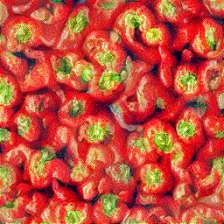}}\hskip.2em
    \subfloat[]{\includegraphics[height=0.24\linewidth]{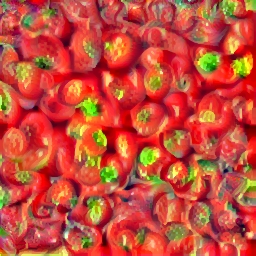}} \\
   \caption{Synthesis results for three example texture photographs. \textbf{(a)} Reference textures. \textbf{(b)} Images synthesized using the method of Portilla \& Simoncelli~\cite{portilla2000parametric}. \textbf{(c)} Images synthesized using Gatys~\etal~\cite{gatys2015texture}.   \textbf{(d)} Images synthesized using our texture model (Eq.~(\ref{eq:texture_syn})).}  
  \label{fig:synthesis}
\end{figure}

\begin{figure}
  \centering
     \subfloat[]{\includegraphics[width=0.325\linewidth]{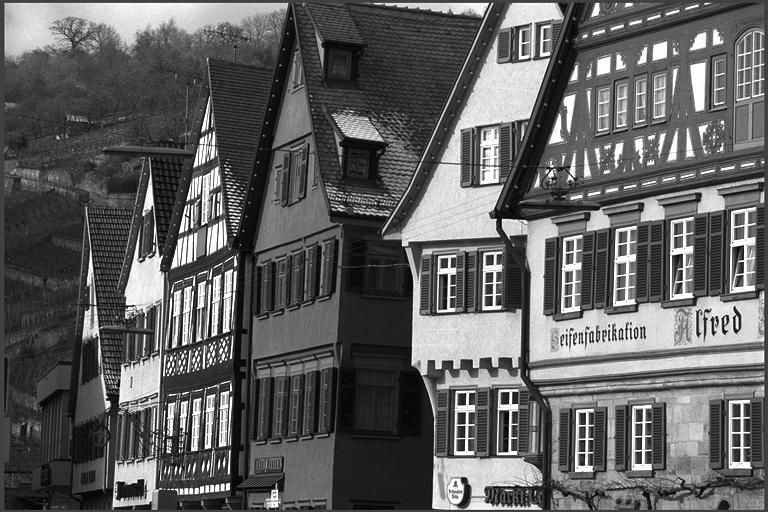}} \hskip.2em
    \subfloat[]{\includegraphics[width=0.325\linewidth]{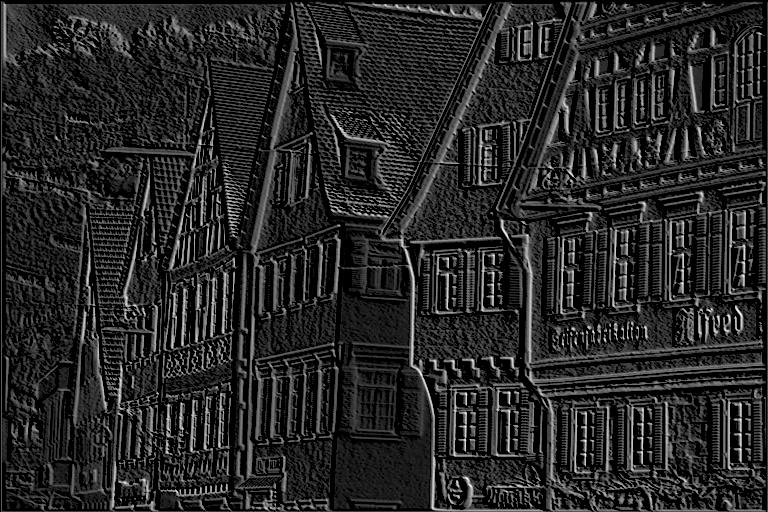}} \hskip.2em
    \subfloat[]{\includegraphics[width=0.325\linewidth]{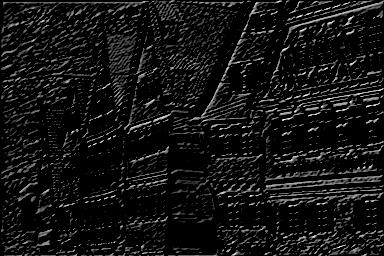}} \\ \vspace{-.25cm}
    \subfloat[]{\includegraphics[width=0.325\linewidth]{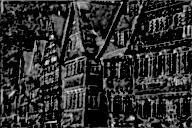}} \hskip.2em
    \subfloat[]{\includegraphics[width=0.325\linewidth]{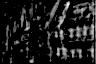}}  \hskip.2em
    \subfloat[]{\includegraphics[width=0.325\linewidth]{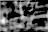}}
  \caption{Selected feature maps from the six layers of the VGG decomposition of the ``buildings'' image. \textbf{(a)} Zeroth stage (original image). \textbf{(b)} First stage. \textbf{(c)} Second stage. \textbf{(d)} Third stage. \textbf{(e)} Fourth stage. \textbf{(f)} Fifth stage. The feature map intensities are re-scaled for better visibility.} 
  \label{fig:activation}
\end{figure}

\subsection{Perceptual Distance Measure}
Next, we specified quality measurements based on $f(x)$ and $f(y)$.  Fig.~\ref{fig:activation} visualizes some feature maps of the six stages of the reference image ``Buildings''. As can been seen,  spatial structures are present at all stages, indicating strong statistical dependencies between neighbouring coefficients. Therefore, use of an $\ell_p$-norm, that assumes statistical independence of errors at different locations, is not appropriate. Inspired by the form of SSIM~\cite{wang2004image}, we defined separate quality measurements for the texture (using the global means) and the structure (using the global correlations) of each pair of corresponding feature maps:
\begin{align}
l(\tilde{x}^{(i)}_j,\tilde{y}^{(i)}_j)=\frac{2 \mu_{\tilde{x}_j}^{(i)}\mu_{\tilde{y}_j}^{(i)}+c_{1}}{\left(\mu_{\tilde{x}_j}^{(i)}\right)^{2}+\left(\mu_{\tilde{y}_j}^{(i)}\right)^{2}+c_{1}},
\label{eq:s1}
\end{align} 
\begin{align}
s(\tilde{x}^{(i)}_j,\tilde{y}^{(i)}_j)=\frac{2 \sigma_{\tilde{x}_j\tilde{y}_j}^{(i)}+c_{2}}{\left(\sigma_{\tilde{x}_j}^{(i)}\right)^{2}+\left(\sigma_{\tilde{y}_j}^{(i)}\right)^{2}+c_{2}},
\label{eq:s2}
\end{align}
where $\mu_{\tilde{x}_j}^{(i)}$, $\mu_{\tilde{y}_j}^{(i)}$, $(\sigma_{\tilde{x}_j}^{(i)})^{2}$, $(\sigma_{\tilde{y}_j}^{(i)})^{2}$, and $\sigma_{\tilde{x}_j\tilde{y}_j}^{(i)}$ represent the global means and variances of $\tilde{x}^{(i)}_j$ and $\tilde{y}^{(i)}_j$, and the global covariance between $\tilde{x}^{(i)}_j$ and $\tilde{y}^{(i)}_j$, respectively. Two small positive constants, $c_{1}$ and $c_{2}$, are included to avoid numerical instability when the denominators are close to zero. The normalization mechanisms in Eq.~(\ref{eq:s1}) and Eq.~(\ref{eq:s2}) serve to equalize the magnitudes of feature maps at different stages. 

Finally, the proposed DISTS model combines the quality measurements from different convolution layers using a weighted sum:
\begin{align}
D(x, y;\alpha,\beta)=1-\sum_{i=0}^{m}\sum_{j=1}^{n_i}\left(\alpha_{ij}l(\tilde{x}^{(i)}_j,\tilde{y}^{(i)}_j)+ \beta_{ij}s(\tilde{x}^{(i)}_j,\tilde{y}^{(i)}_j)\right),
\label{eq:vgg-metric}
\end{align}
where $\{\alpha_{ij}, \beta_{ij}\}$ are positive learnable weights, satisfying $\sum_{i=0}^{m}\sum_{j=1}^{n_i}(\alpha_{ij} +\beta_{ij})=1$.  Note that the convolution kernels are fixed throughout the development of the method. Fig.~\ref{fig:vgg} shows the full computation diagram of our quality assessment system. 

\begin{lemma}For $\forall$ $\tilde{x}^{(i)}_j,\tilde{y}^{(i)}_j \in \R^{n}_{+}$ (as is the case for responses after ReLU nonlinearity), it can be shown that
\begin{align}
    d(x,y)=\sqrt{D(x,y)}
\end{align}
is a proper metric, satisfying 
\begin{itemize}
\item non-negativity: $d(x, y) \ge 0$; 
\item symmetry: $d(x, y) = d(y, x)$; 
\item  triangle inequality: $d(x, z) \le d(x, y) + d(y, z)$;
\item  identity of indiscernibles (\ie, unique minimum): $d(x, y) = 0 \Leftrightarrow x = y$.
\end{itemize}
\end{lemma}
\begin{proof}
 The non-negative and symmetric properties are immediately apparent. The identity of indiscernibles is guaranteed due to the injective mapping function and the use of SSIM-motivated quality measurements. To verify the triangle inequality, we first rewrite $d(x,y)$ as 
 \begin{align}
     d(x,y) = \sqrt{\sum_{i=0}^m\sum_{j=1}^{n_i}d_{ij}^2(x,y)},
 \end{align}
 where 
  \begin{align}
     d_{ij}(x,y) = \sqrt{\alpha_{ij}(1-l(\tilde{x}^{(i)}_j,\tilde{y}^{(i)}_j))+\beta_{ij}(1 - s(\tilde{x}^{(i)}_j,\tilde{y}^{(i)}_j))}.
 \end{align}
 Brunet \etal~\cite{brunet2011mathematical} have proved that $d_{ij}(x,y)$ is a metric for $\alpha_{ij}\ge 0$ and $\beta_{ij}\ge 0$. Then,
  \begin{align}
     d(x,y) &\le \sqrt{\sum_{i,j}(d_{ij}(x,z)+d_{ij}(z,y))^2}\\
     &\le\sqrt{\sum_{i,j}d_{ij}^2(x,z)}+ \sqrt{\sum_{i,j}d_{ij}^2(y,z)}\label{eq:csi}\\
     &=d(x,z) + d(z,y),
 \end{align}
 where Eq.~(\ref{eq:csi}) follows from the Cauchy–Schwarz inequality.
\end{proof}



\begin{figure}[t]
  \centering
    \includegraphics[width=0.9\linewidth]{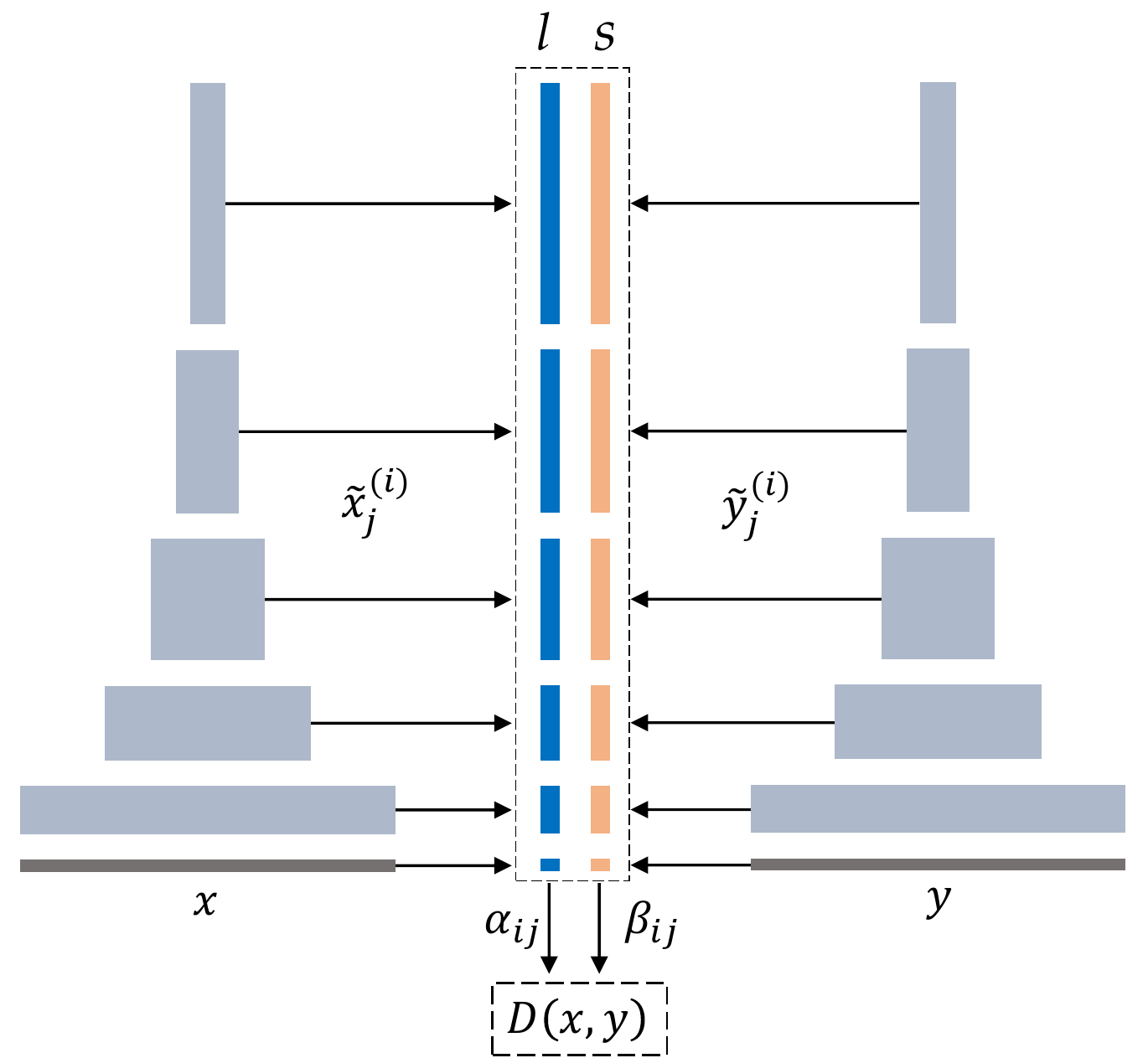}
    \vspace{-.25cm}   
  \caption{VGG-based perceptual representation for the proposed DISTS model. It contains  a total of six stages (including the zeroth stage of raw pixels), and the numbers of feature maps at each stage are $3$, $64$, $128$, $256$, $512$ and $512$, respectively. Global texture and structure similarity measurements are made at each stage, and combined with a weighted summation, giving rise to the final model defined in Eq.~(\ref{eq:vgg-metric}).}  
  \label{fig:vgg}
\end{figure}

\subsection{Model Training} \label{sec:train}
The perceptual weights $\{\alpha, \beta\}$ in Eq.~(\ref{eq:vgg-metric}) were jointly optimized for human perception of image quality and texture invariance. Specifically, for image quality, we minimized the absolute error between model predictions and human ratings:
\begin{align}
E_1(x,y;
\alpha,\beta)=\vert D(x, y;
\alpha,\beta)-q(y))\vert,
\label{eq:loss1}
\end{align}
where $q(y)$ denotes the normalized ground-truth quality score of $y$ collected from psychophysical experiments. We chose the large-scale IQA dataset KADID-10k~\cite{lin2019kadid} as the training set, which contains $81$ reference images, each of which is distorted by $25$ distortion types at $5$ distortion levels. In addition, we explicitly enforced the model to be invariant to texture substitution in a data-driven fashion. We minimized the distance (measured by Eq.~(\ref{eq:vgg-metric})) between two patches $(z_1, z_2)$ sampled from the same texture image $z$:
\begin{align}
E_{2}(z;
\alpha,\beta)=D(z_1, z_2;
\alpha,\beta).
\label{eq:loss2}
\end{align}
We selected texture images from the describable textures dataset (DTD)~\cite{cimpoi2014describing},  consisting of $5,640$ images ($47$ categories and $120$ images for each category). In practice, we randomly sampled two minibatches $\mathcal{Q}$ and $\mathcal{T}$ from   KADID-10k and DTD, respectively, and used a variant of stochastic gradient descent to adjust the parameters $\{\alpha,\beta\}$:
\begin{align}
   E(\mathcal{Q},\mathcal{T};\alpha,\beta) = \frac{1}{\vert\mathcal{Q}\vert} \sum_{x,y\in \mathcal{Q}}E_{1}(x,y;
\alpha,\beta) +\lambda\frac{1}{\vert\mathcal{T}\vert} \sum_{z\in \mathcal{T}}E_{2}(z;
\alpha,\beta)
\label{eq:loss}
\end{align}
where $\lambda$ governs the trade-off between the two terms.  

\begin{table*}[t]
  \centering
  \caption{ Performance comparison on three standard IQA databases. Larger PLCC, SRCC and KRCC values indicate better performance. CNN-based methods are highlighted in italics}
  \setlength{\tabcolsep}{2.2mm}{
    \begin{tabular}{lccccccccc}
    \toprule
      \multirow{2}{*}[-3pt]{Method} & \multicolumn{3}{c}{LIVE~\cite{LIVE}}&\multicolumn{3}{c}{CSIQ~\cite{larson:011006}}&\multicolumn{3}{c}{TID2013~\cite{Ponomarenko201557}}\\ 
      \cmidrule(lr){2-4} \cmidrule(lr){5-7} \cmidrule(lr){8-10} 
      &PLCC & SRCC & KRCC  & PLCC & SRCC & KRCC  & PLCC & SRCC & KRCC  \\ \hline 
     PSNR  & 0.865 & 0.873 & 0.680  & 0.819 & 0.810 & 0.601  & 0.677 & 0.687 & 0.496\\
     SSIM~\cite{wang2004image}  & 0.937 &0.948 & 0.796 & 0.852 & 0.865 & 0.680 & 0.777 & 0.727 & 0.545\\
     MS-SSIM~\cite{wang2003multiscale}  &  0.940 & 0.951 & 0.805  &  0.889 & 0.906 & 0.730 & 0.830 & 0.786 &  0.605\\
     VSI~\cite{zhang2014vsi}  & 0.948 & 0.952 & 0.806 & 0.928 & 0.942 & 0.786  & \textbf{0.900} & \textbf{0.897} & \textbf{0.718}\\
     MAD~\cite{larson:011006}  & \textbf{0.968} & \textbf{0.967} & \textbf{0.842}  & \textbf{0.950} & \textbf{0.947} & \textbf{0.797} & 0.827 & 0.781 & 0.604\\
     VIF~\cite{sheikh2006image}  & 0.960 & 0.964 & 0.828 & 0.913 & 0.911 & 0.743 &  0.771   & 0.677 &  0.518\\
     FSIM$_\mathrm{c}$~\cite{zhang2011fsim}  & \textbf{0.961} &\textbf{0.965} & \textbf{0.836} &	0.919 & 0.931 &	0.769  & \textbf{0.877} & 0.851 & 0.667\\
     NLPD~\cite{laparra2016perceptual}  & 0.932 & 0.937 & 0.778 & 0.923 & 0.932 & 0.769 & 0.839 &  0.800  & 0.625\\
     GMSD~\cite{xue2014gradient}  & 0.957 & 0.960 & 0.827 & \textbf{0.945}  &  \textbf{0.950} & \textbf{0.804} & 0.855 & 0.804 & 0.634\\
     \hline
     \textit{DeepIQA}~\cite{bosse2018deep}  & 0.940 & 0.947 & 0.791 & 0.901 & 0.909 & 0.732 & 0.834 & 0.831  & 0.631\\
     \textit{PieAPP}~\cite{prashnani2018pieapp}  & 0.908 & 0.919 & 0.750 & 0.877 & 0.892 & 0.715 & 0.859 & \textbf{0.876} & \textbf{0.683}\\
     \textit{LPIPS}~\cite{zhang2018unreasonable}  & 0.934 & 0.932 & 0.765 & 0.896 & 0.876 & 0.689 & 0.749 & 0.670 & 0.497\\
     \hline
     \textit{DISTS (ours)} & 0.954 & 0.954 & 0.811 & 0.928 & 0.929 & 0.767& 0.855 & 0.830 & 0.639\\
    \bottomrule
    \end{tabular}}
  \label{tab:iqa_database}
\end{table*}

\subsection{Connections to Other Full-Reference IQA Methods}
The proposed DISTS model has a close relationship to a number of existing IQA methods.
\begin{itemize}
\item \textit{SSIM and its variants}~\cite{wang2004image,wang2003multiscale,wang2005translation}: The multi-scale extension of SSIM~\cite{wang2003multiscale} incorporates the variations of viewing conditions in IQA, and calibrates the cross-scale parameters via subjective testing on artificially synthesized images. Our model follows a similar approach, building on a multi-scale hierarchical representation and directly calibrating cross-scale parameters (\ie, $\alpha,\beta$) using subject-rated natural images with various distortions. The extension of SSIM into the complex wavelet domain~\cite{wang2005translation} gains invariance to small geometric transformations by measuring relative phase patterns of the wavelet coefficients. As we show in Section~\ref{sec:geo}, by optimizing for texture invariance, DISTS inherits insensitivity to mild geometric transformations. It is worth noting that unlike SSIM and its variants, DISTS is based on global spatial statistics, and thus does not provide a spatial map of quality.
\item \textit{The adaptive linear system framework}~\cite{wang2005adaptive} decomposes the distortion between
two images into a linear combination of components that are adapted to local image structures, separating structural and non-structural distortions. It generalizes many IQA models, including MSE, space/frequency weighting~\cite{mannos1974effects,watson1997visibility}, transform domain masking~\cite{teo1994perceptual}, and the tangent distance~\cite{simard1998transformation}. 
DISTS can be seen as an adaptive nonlinear system, where structure comparison  captures structural distortions, and texture comparison measures non-structural distortions, with basis functions adapted to global image content. 
\item \textit{Style and content separation}~\cite{gatys2016image} based on the pre-trained VGG network has reignited the field of style transfer. Specifically, the style loss is built upon the correlations between convolution responses at the same stages (\ie, the Gram matrix) while the content loss is defined by the MSE between the two representations. These two components are redundant, and the combined loss does not have the desired property of unique minima we seek. 
\comment{By incorporating the input image as the zeroth stage feature representation of VGG and making SSIM-inspired quality measurements, the square root of DISTS is a valid metric. }
\item \textit{Image restoration losses}~\cite{johnson2016perceptual} in the era of deep learning are typically defined as a weighted sum of $\ell_p$-norm distances computed on the raw pixels and several stages of VGG feature maps, where the weights are manually tuned for tasks at hand. Later stages of the VGG representation are often preferred so as to incorporate image semantics into low-level vision, encouraging perceptually meaningful details that are not necessarily aligned with the underlying image. This type of loss does not achieve the level of texture invariance we are looking for. \comment{Moreover, the weights of DISTS are jointly optimized for image quality and texture invariance, and can be used across multiple low-level vision tasks.}
\end{itemize}

\begin{figure*}[t]
  \centering
  \begin{tabular}{ccc}
    \multirow{2}{*}[120pt]{\subfloat{\includegraphics[height=0.45\linewidth]{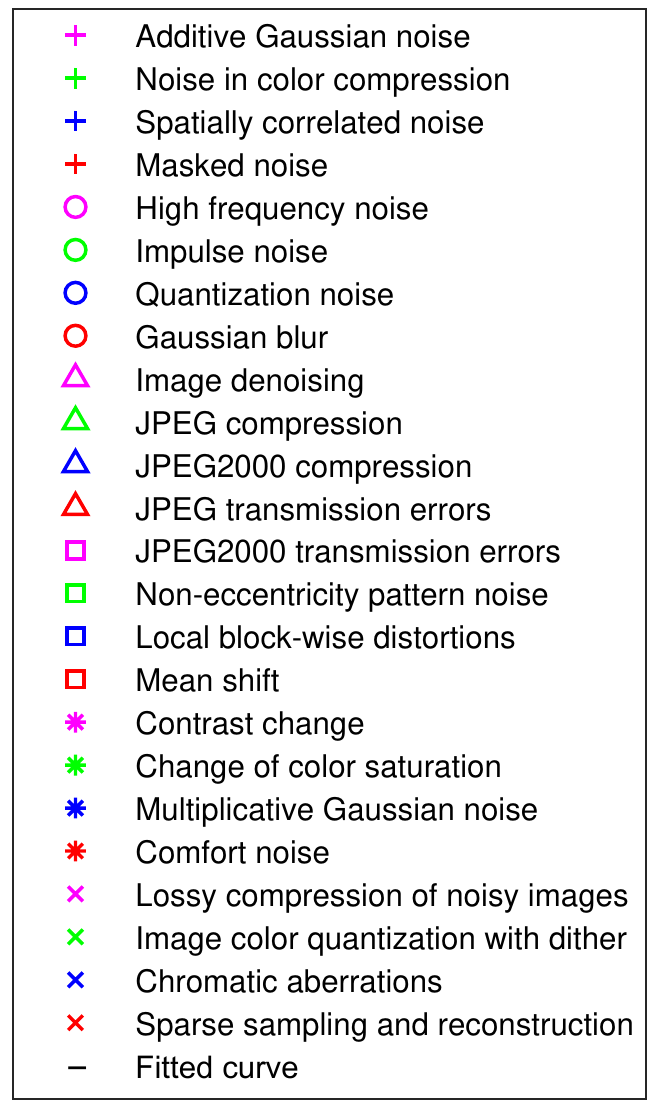}}} & 
    \subfloat{\includegraphics[height=0.27\linewidth]{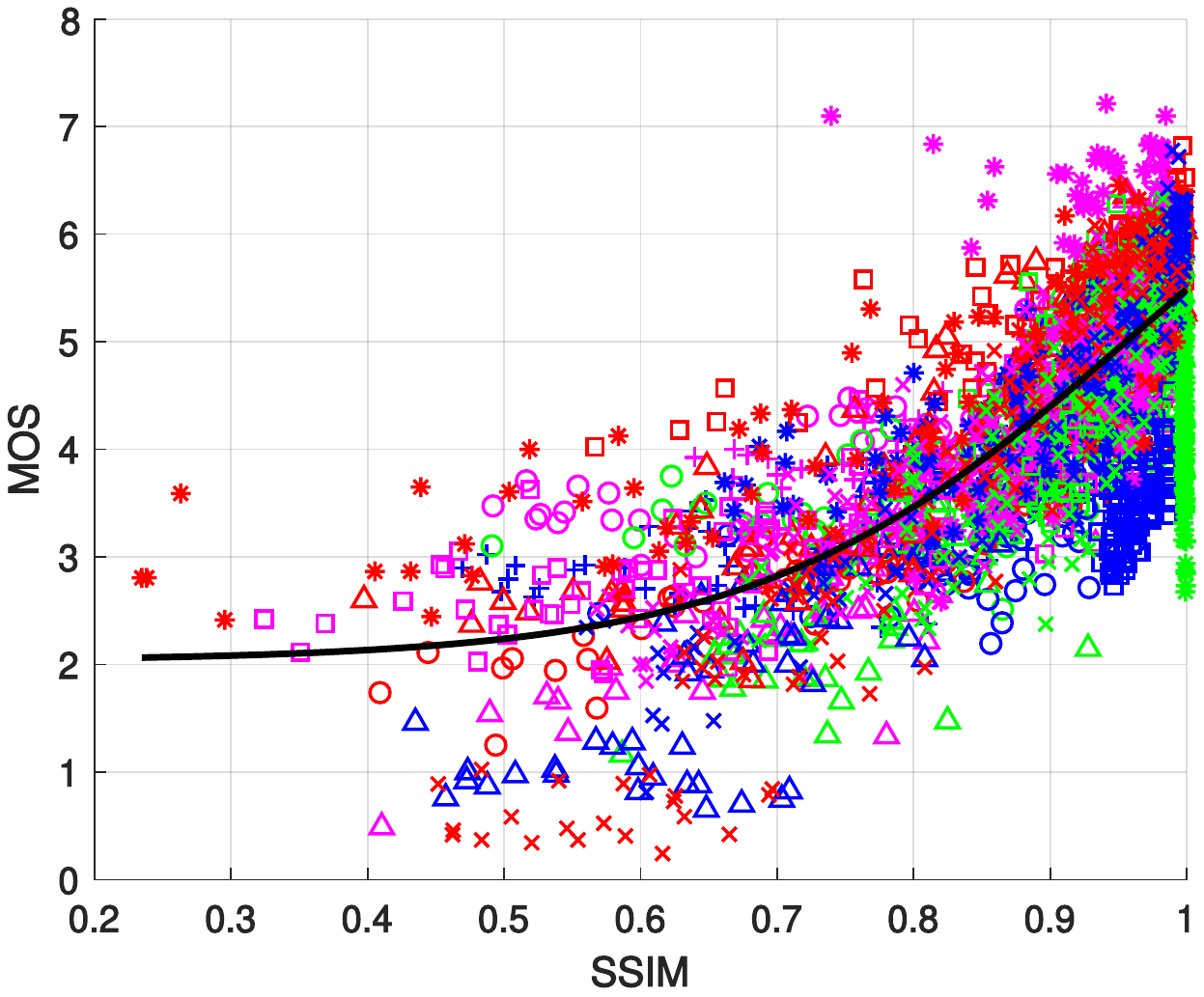}}  \hskip.2em
    \subfloat{\includegraphics[height=0.27\linewidth]{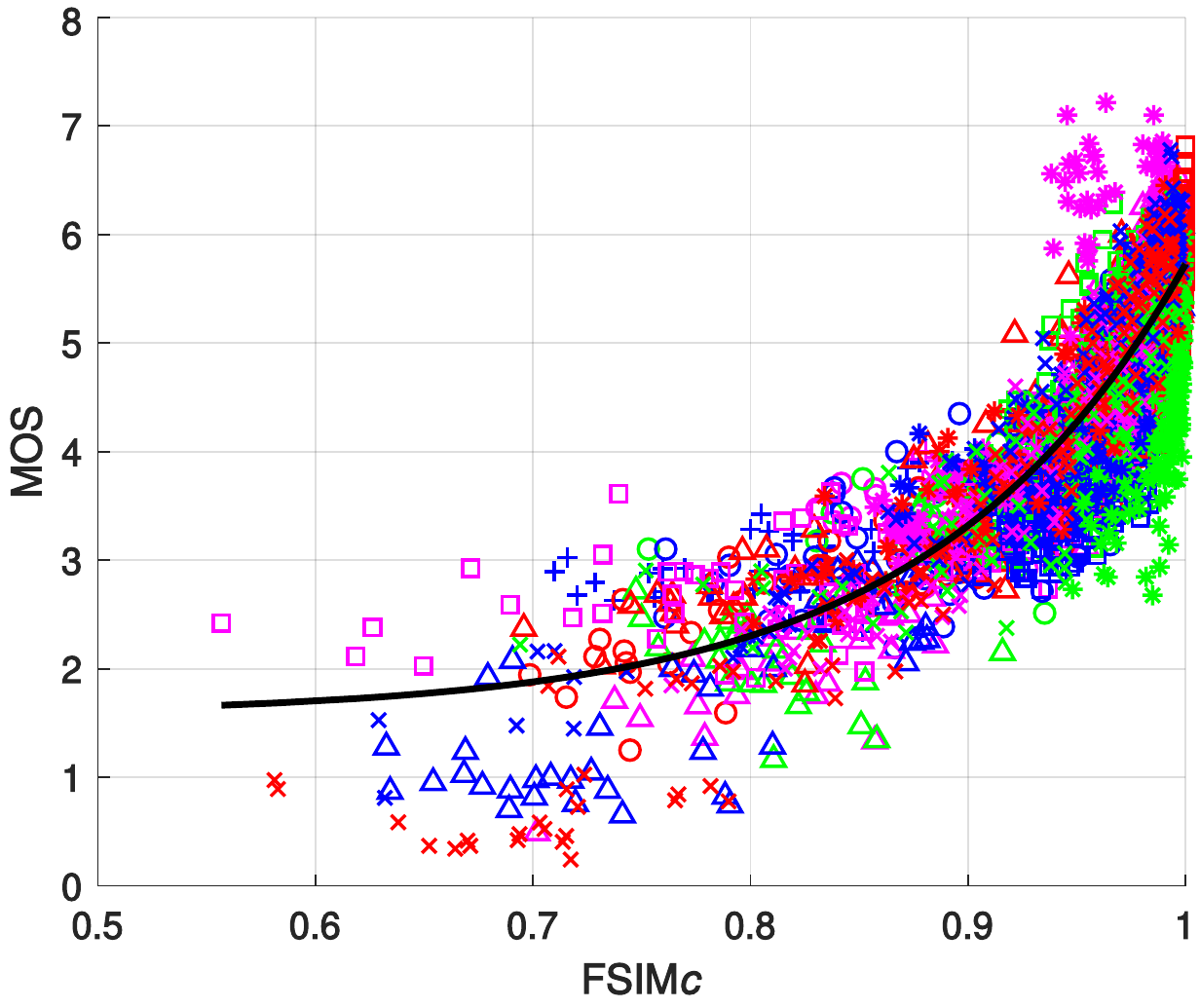}}\\
    &
    \subfloat{\includegraphics[height=0.27\linewidth]{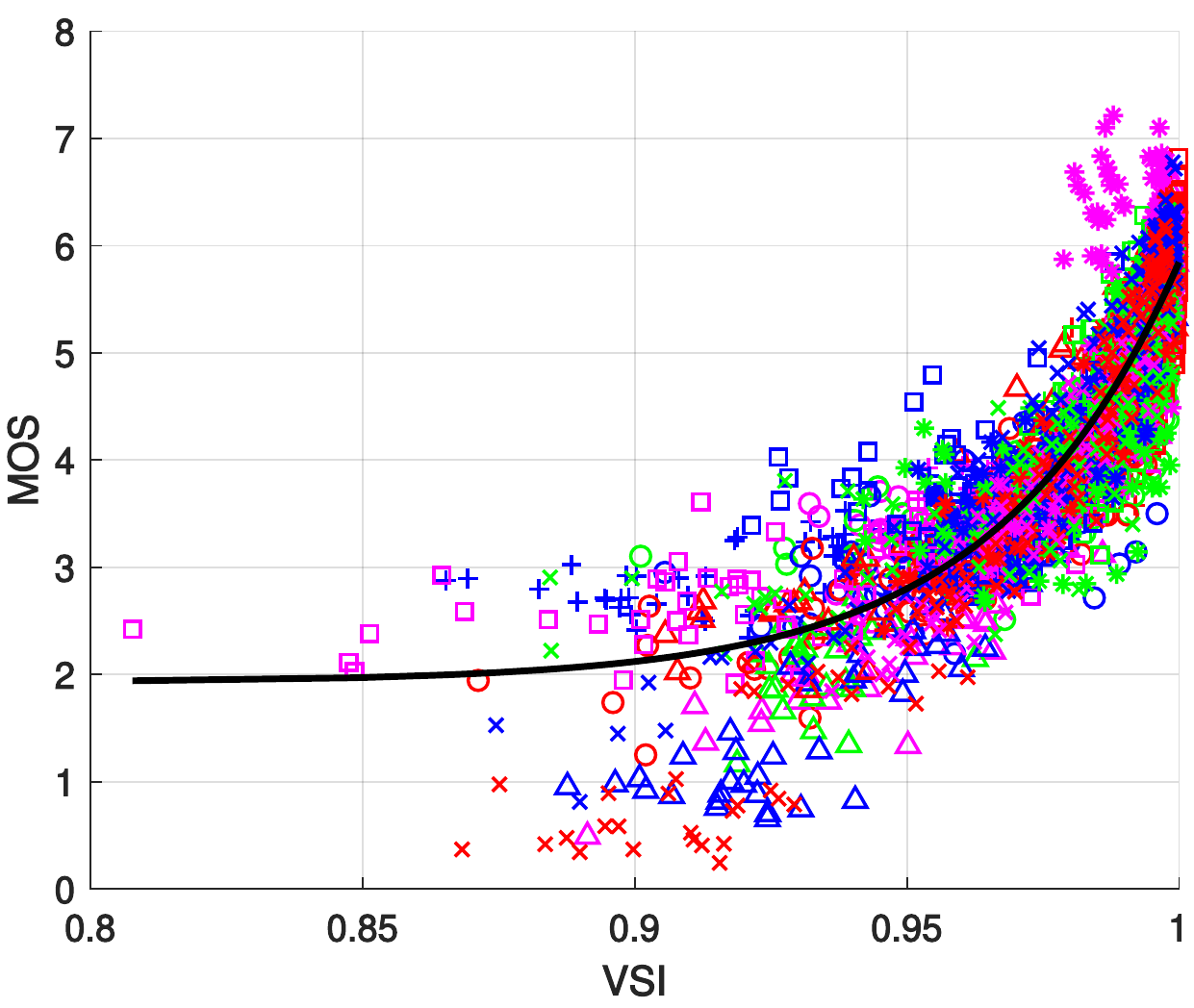}} \hskip.2em 
    \subfloat{\includegraphics[height=0.27\linewidth]{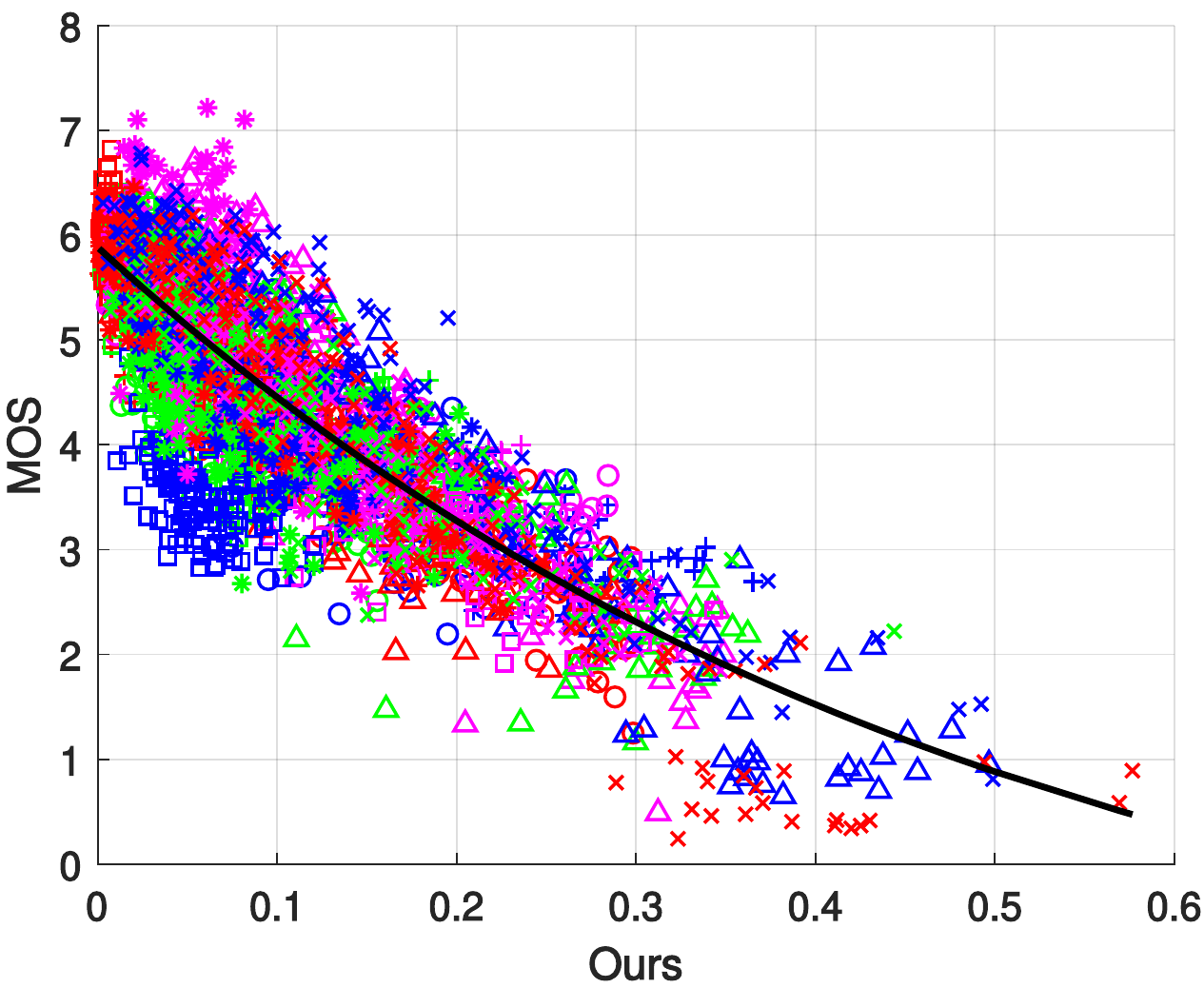}} 
    \end{tabular}
  \caption{Comparison of human mean opinion scores (MOSs) against SSIM, FSIM$_\mathrm{c}$, VSI, and DISTS (ours) on the TID2013 database. \comment{The fitted curve of DISTS is slightly more linear than the others.}}
  \label{fig:scatter}
\end{figure*}

\section{Experiments}

In this section, we  present the implementation details of the proposed DISTS. We then compare our method with a wide range of image similarity models in terms of quality prediction, texture similarity, texture classification/retrieval, and invariance of geometric transformations.

\subsection{Implementation Details}\label{subsec:id}
We fixed the filter kernels of the pre-trained VGG, and learned the perceptual weights $\{\alpha,\beta\}$. The training was carried out by optimizing the objective function in Eq.~(\ref{eq:loss}), assuming a value of $\lambda=1$, using Adam~\cite{kingma2014adam} with a batch size of 32 and an initial learning rate of $1\times10^{-4}$. After every 1K iterations, we reduced the learning rate  by a factor of $2$. We trained DISTS for 5K iterations, which takes approximately one hour on an NVIDIA GTX 2080 GPU. To ensure a unique minimum of our model, we projected the weights of the zeroth stage onto the interval $[0.02, 1]$ after each gradient step. We chose a $5\times 5$ Hanning window to reduce subsampling-induced aliasing in the VGG representation. Both $c_1$ in Eq.~(\ref{eq:s1}) and $c_2$ in Eq.~(\ref{eq:s2}) were set to $10^{-6}$. During training and testing, we followed the suggestions in~\cite{wang2004image}, and re-scaled the input images such that the smaller dimension has $256$ pixels. \added{The size of texture patches as input to Eq.~(\ref{eq:loss2}) was $256\times256\times3$, cropped from the same texture images.}

\subsection{Performance on Quality Prediction} 
After training on the entire KADID dataset~\cite{lin2019kadid}, DISTS was tested on the other three standard IQA databases LIVE~\cite{LIVE}, CSIQ~\cite{larson:011006} and TID2013~\cite{Ponomarenko201557}. We used the Pearson linear correlation coefficient (PLCC), the Spearman rank correlation coefficient (SRCC), and the Kendall rank correlation coefficient (KRCC) as  evaluation criteria. Before computing PLCC, we fitted a four-parameter function to allow and compensate for a smooth
nonlinear relationship:
\begin{align}
\hat{D}=\left(\eta_{1}-\eta_{2}\right) /\left(1+\exp \left(-\left(D-\eta_{3}\right) /\left|\eta_{4}\right|\right)\right)+\eta_{2},
\label{eq:nonlinearMap}
\end{align}
where $\{\eta_{i}\}_{i=1}^4$ are parameters. 
We compared DISTS against a set of full-reference IQA methods, including nine knowledge-driven models and three data-driven CNN-based models. The implementations of all methods were obtained from the respective authors, except for DeepIQA~\cite{bosse2018deep}, which was retrained on KADID for fair comparison. As LPIPS~\cite{zhang2018unreasonable} has different configurations, we chose the default one (known as \textit{LPIPS-VGG-lin}). 

Results, reported in Table \ref{tab:iqa_database}, demonstrate that DISTS performs favorably in comparison to both classic methods (\eg, PSNR and SSIM~\cite{wang2004image}) and CNN-based models (\eg, DeepIQA~\cite{bosse2018deep} and LPIPS~\cite{zhang2018unreasonable}). Overall, the best performances across all three databases and all comparison metrics are obtained with MAD~\cite{larson:011006}, FSIM$_\mathrm{c}$~\cite{zhang2011fsim} and GMSD~\cite{xue2014gradient}. It is worth noting that these three databases have been re-used for many years throughout the algorithm design processes, and recent full-reference IQA methods may be unintentionally over-adapting via extensive computational module selection, raising the risk of over-fitting (see Fig.~\ref{fig:optim}).
Fig.~\ref{fig:scatter} shows scatter plots of raw model predictions of representative IQA methods versus  subjective mean opinion scores (MOSs)  on the TID2013 database. From the fitted functions (Eq.~(\ref{eq:nonlinearMap})), one can observe that DISTS is nearly linear in MOS.

We also tested DISTS on  BAPPS~\cite{zhang2018unreasonable}, a large-scale and highly-varied patch similarity dataset. BAPPS contains  traditional synthetic distortions, such as geometric and photometric manipulation, noise contamination, blurring and compression, CNN-based distortions (e.g., from denoising autoencoders and image restoration tasks), and distortions generated by real-world image processing systems. The human judgments are obtained from a two-alternative forced choice (2AFC) experiment. \added{The evaluation criterion is the 2AFC score \cite{zhang2018unreasonable}, which quantifies the proportion of human agreement with the IQA model, computed as $p\hat{p}+(1-p)(1-\hat{p})$, where $p$ is the percentage of human choices in favor of a given image in each pair, and $\hat{p}\in\{0,1\}$ is the preference of the IQA model. Larger values indicate better agreement between model predictions and human judgments.} Results are compiled in Table \ref{tab:bapps}, showing that DISTS (which was not trained on BAPPS, or any similar database) achieves  a comparable performance to LPIPS~\cite{zhang2018unreasonable} (which was trained on BAPPS). We conclude that DISTS predicts image quality well,  and generalizes well to challenging unseen distortions, such as those caused by real-world algorithms.

\begin{table*}
\newcommand{\tabincell}[2]{\begin{tabular}{@{}#1@{}}#2\end{tabular}}
  \centering
  \caption{Performance comparison of various IQA methods on the BAPPS~\cite{zhang2018unreasonable} dataset using the 2AFC score, which quantifies the agreement with human judgments.  Values lie in the range $[0,1]$, with a higher value indicating better agreement}
    \begin{tabular}{lccccccccc}
    \toprule 
     \multirow{2}{*}[-7pt]{Method} & \multicolumn{3}{c}{Synthetic distortions}&\multicolumn{5}{c}{Distortions by real-world algorithms}&\multirow{2}{*}[-7pt]{All} \\ 
    \cmidrule(lr){2-4} \cmidrule(lr){5-9}
    &Traditional & CNN-based& All & \tabincell{c}{Super\\resolution} & \tabincell{c}{Video\\ deblurring} & Colorization & \tabincell{c}{Frame\\interpolation}& All \\ \hline 
     Human  & 0.808 & 0.844 & 0.826 & 0.734 & 0.671 & 0.688 & 0.686 & 0.695 & 0.739 \\  
     \hline
     PSNR  & 0.573 & 0.801 & 0.687 & 0.642 & 0.590 & 0.624 &  0.543&  0.614&  0.633 \\
     SSIM~\cite{wang2004image} & 0.605 & 0.806 & 0.705 & 0.647 & 0.589 & 0.624 & 0.573 & 0.617 & 0.640 \\
     MS-SSIM~\cite{wang2003multiscale}  & 0.585 & 0.768 & 0.676 & 0.638 & 0.589 & 0.524 & 0.572 & 0.596 & 0.617\\
     VSI~\cite{zhang2014vsi}  & 0.630 & 0.818 & 0.724 & 0.668 & 0.592 & 0.597 & 0.568 & 0.622 & 0.648 \\
     MAD~\cite{larson:011006}  & 0.598 & 0.770 & 0.684 & 0.655 & 0.593 & 0.490 & 0.581 & 0.599 & 0.621\\
     VIF~\cite{sheikh2006image}  & 0.556 & 0.744 & 0.650 & 0.651 & 0.594 & 0.515 & 0.597 & 0.603 & 0.615\\
     FSIM$_\mathrm{c}$~\cite{zhang2011fsim} & 0.627 & 0.794 & 0.710 & 0.660 & 0.590 & 0.573 & 0.581 & 0.615 & 0.640\\
     NLPD~\cite{laparra2016perceptual}  & 0.550 & 0.764 & 0.657 & 0.655 & 0.584 & 0.528 & 0.552 & 0.600 & 0.615\\
     GMSD~\cite{xue2014gradient}  & 0.609 & 0.772 & 0.690 & 0.677 & 0.594 & 0.517 & 0.575 & 0.613 & 0.633\\
     \hline
     \textit{DeepIQA}~\cite{bosse2018deep}  &  0.703 &    0.794  &  0.748 &  0.660 & 0.582 & 0.585 & 0.598 & 0.615 & 0.650 \\
     \textit{PieAPP}~\cite{prashnani2018pieapp} & 0.727 & 0.770 & 0.746 & 0.684 & 0.585 & 0.594 & 0.598  &  0.627 & 0.659 \\
     \textit{LPIPS}~\cite{zhang2018unreasonable}  & \textbf{0.760} & \textbf{0.828} & \textbf{0.794} & \textbf{0.705} & \textbf{0.605} & \textbf{0.625} & \textbf{0.630 }& \textbf{0.641} & \textbf{0.692}\\
     \hline
     \textit{DISTS (ours)}& \textbf{0.772} & \textbf{0.822} & \textbf{0.797} & \textbf{0.710} & \textbf{0.600} & \textbf{0.627} &  \textbf{0.625}  & \textbf{0.651} & \textbf{0.689} \\
    \bottomrule
    \end{tabular}
  \label{tab:bapps}
\end{table*}

\subsection{Performance on Texture Similarity}
We also tested the performance of DISTS on texture quality assessment. Since most knowledge-driven full-reference IQA models are not good at measuring texture similarity (see Fig~\ref{fig:texturefailure}), we only included a subset for reference. To these we added CW-SSIM~\cite{wang2005translation} and three computational models specifically designed for texture similarity - STSIM~\cite{zujovic2013structural}, NPTSM~\cite{alfarraj2016content} and IGSTQA~\cite{golestaneh2018synthesized}. STSIM is available in several configurations, and we chose local STSIM-2 that is publicly available\footnote{\url{https://github.com/andreydung/Steerable-filter}}. 

We used a synthesized texture quality assessment database SynTEX~\cite{golestaneh2015effect}, consisting of $21$ reference textures with $105$ synthesized versions  generated by five texture synthesis algorithms. 
Table~\ref{tab:syntex} shows the results of correlation coefficients, where we can see that texture similarity models generally perform better than IQA models. Focusing on texture similarity,
IGSTQA~\cite{golestaneh2018synthesized} achieves a relatively high performance,  but is still inferior to DISTS. This indicates that the VGG-based global measurements of DISTS capture the essential features and attributes of visual textures.

To further test the capabilities of DISTS in quantifying texture distortions, we constructed a texture quality database (TQD), based on 10 texture images selected from Pixabay\footnote{\url{https://pixabay.com/images/search/texture}}. Each texture image was corrupted with seven traditional synthetic distortions: additive white Gaussian noise, Gaussian blur, JPEG compression, JPEG2000 compression, pink noise, chromatic aberration, and image color quantization. For each distortion type, we randomly selected one distortion level from a set of three levels, and applied it to each texture image. We then created four copies of each texture using different texture synthesis algorithms, including two classical ones (a parametric model~\cite{portilla2000parametric} and a non-parametric model~\cite{efros1999texture}) and two CNN-based algorithms ~\cite{gatys2015texture,snelgrove2017high}. Last, to produce ``high-quality'' images, we randomly cropped four subimages from each of the original textures. In total, TQD has $10 \times 15$ images. We gathered human data from $10$ subjects, who had general knowledge of image processing but were unaware of the detailed purpose of the study. The viewing distance was fixed to enforce a visual resolution $32$ pixels per degree of visual angle.  Each subject was shown all ten sets of images, one set at a time, starting with the reference image, and was asked to rank the images according to their perceptual similarity to the reference. 
Rather than simply averaging the human opinions, we used reciprocal rank
fusion~\cite{cormack2009reciprocal} to obtain the final ranking
\begin{align}
    r(x)=\sum_{k=1}^{K} \frac{1}{\gamma+r_{k}(x)},
\end{align}
where $r_{k}(x)$ is the rank of $x$ given by the $k$-th subject and $\gamma$ is an additive constant that helps to mitigate the impact of outliers~\cite{cormack2009reciprocal}. 
Table~\ref{tab:syntex} lists the results, where we computed the correlations within each texture pattern and averaged them across textures. We found that nearly all existing models perform poorly on the new database, including those tailored for texture similarity. In contrast, DISTS significantly outperforms these methods by a large margin. Fig.~\ref{fig:rank} shows a set of texture examples, where we noticed that DISTS gives high rankings to resampled images and low rankings to images suffering from visible distortions. This demonstrates that DISTS is in close agreement with human perception of texture quality, and suggests  potential uses in other texture analysis problems, such as high-quality texture retrieval.

\begin{table}[t]
  \centering
  \caption{Performance comparison on two texture quality databases. Texture similarity models are highlighted in italics}
   \setlength{\tabcolsep}{1.6mm}{
    \begin{tabular}{lcccccc}
    \toprule
      \multirow{2}{*}[-3pt]{Method} & \multicolumn{3}{c}{SynTEX~\cite{golestaneh2015effect}} & \multicolumn{3}{c}{TQD (proposed)} \\ \cmidrule(lr){2-4} \cmidrule(lr){5-7}
    & PLCC & SRCC & KRCC & PLCC & SRCC & KRCC \\ \hline 
     SSIM~\cite{wang2004image}  & 0.619 & 0.620 & 0.446 & 0.330 & 0.307 & 0.185 \\
     CW-SSIM~\cite{wang2005translation}  & 0.532 & 0.497 & 0.335 & 0.344 & 0.325 & 0.238 \\
     DeepIQA~\cite{bosse2018deep} & 0.550 & 0.512 & 0.354 & 0.458 & 0.444 & 0.323 \\
     PieAPP~\cite{prashnani2018pieapp} & 0.719 & 0.715 & 0.532 & 0.721 & 0.718 & 0.556 \\
     LPIPS~\cite{zhang2018unreasonable} & 0.674 & 0.663 & 0.478 &  0.402 & 0.392 & 0.301 \\
     \textit{STSIM}~\cite{zujovic2013structural} & 0.650 & 0.643 & 0.469 & 0.422 & 0.408 & 0.315\\
     \textit{NPTSM}~\cite{alfarraj2016content} & 0.505 & 0.496 & 0.361 & 0.678 &  0.679 & 0.547 \\ 
     \textit{IGSTQA}~\cite{golestaneh2018synthesized} & \textbf{0.816} &  \textbf{0.820} & \textbf{0.621} &  \textbf{0.804} &  \textbf{0.802} & \textbf{0.651} \\
     DISTS (ours) & \textbf{0.901} &  \textbf{0.923} & \textbf{0.759} & \textbf{0.903} & \textbf{0.910} & \textbf{0.785} \\
    \bottomrule
    \end{tabular}}
  \label{tab:syntex} 
\end{table}

\begin{figure*}[t]
  \centering
    \subfloat[]{\includegraphics[height=0.18\linewidth]{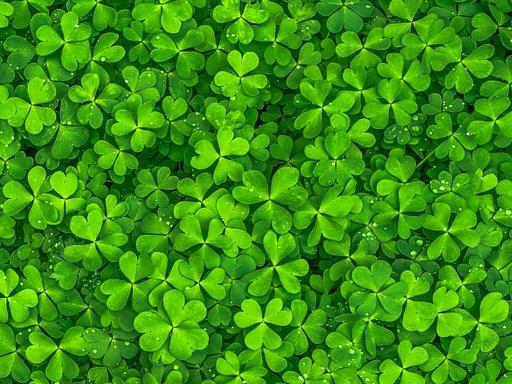}}  \hskip.2em
    \subfloat[]{\includegraphics[height=0.18\linewidth]{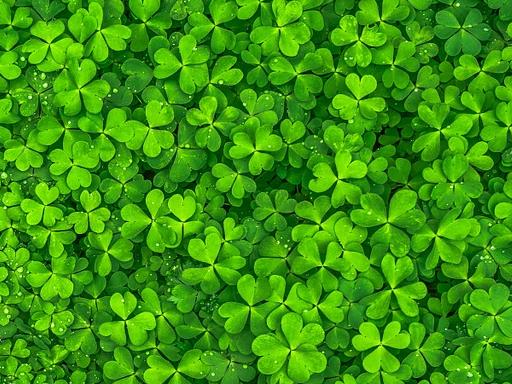}}  \hskip.2em
    \subfloat[]{\includegraphics[height=0.18\linewidth]{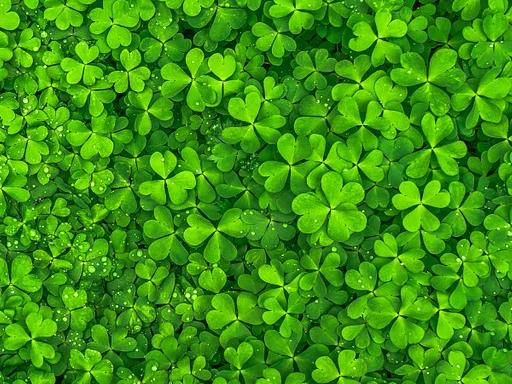}}  \hskip.2em
    \subfloat[]{\includegraphics[height=0.18\linewidth]{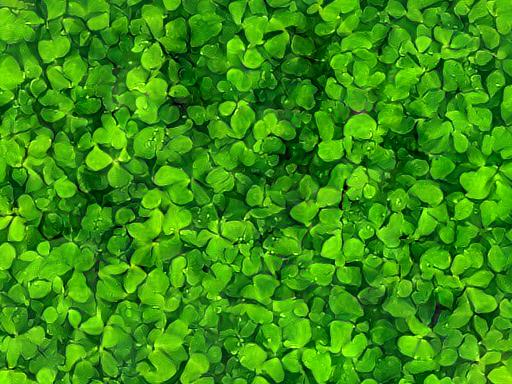}}  \\ \vspace{-0.25cm}
    \subfloat[]{\includegraphics[height=0.18\linewidth]{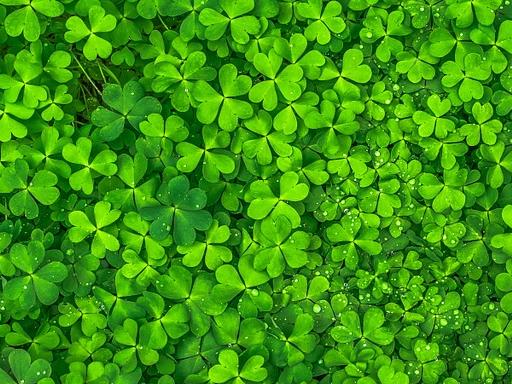}} \hskip.2em
    \subfloat[]{\includegraphics[height=0.18\linewidth]{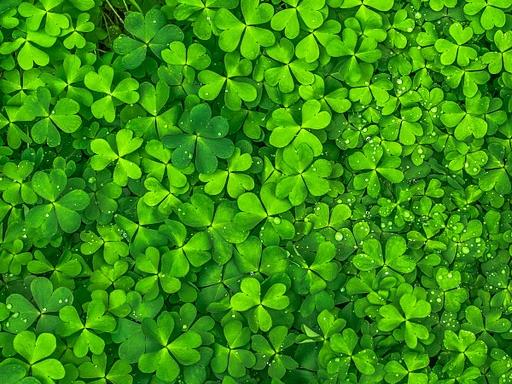}}  \hskip.2em
    \subfloat[]{\includegraphics[height=0.18\linewidth]{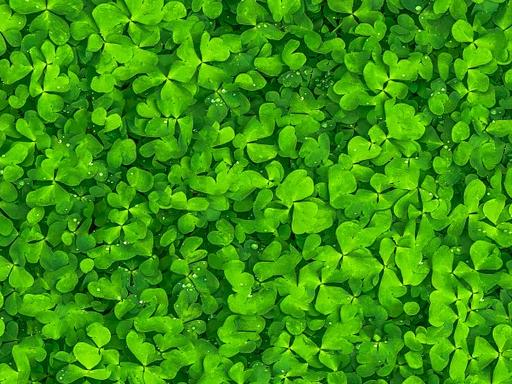}}  \hskip.2em
    \subfloat[]{\includegraphics[height=0.18\linewidth]{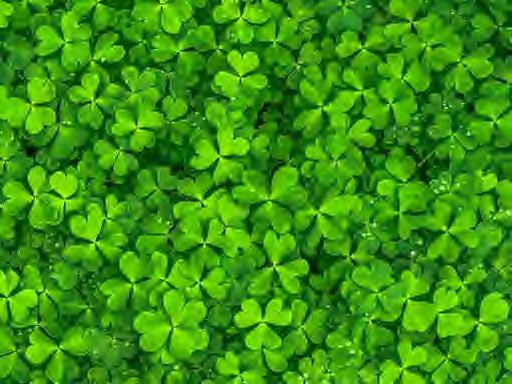}} \\ \vspace{-0.25cm}
    \subfloat[]{\includegraphics[height=0.18\linewidth]{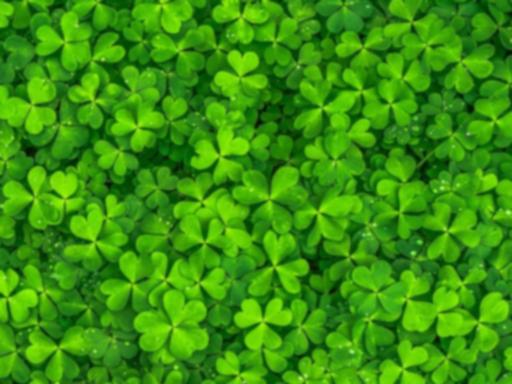}}  \hskip.2em
    \subfloat[]{\includegraphics[height=0.18\linewidth]{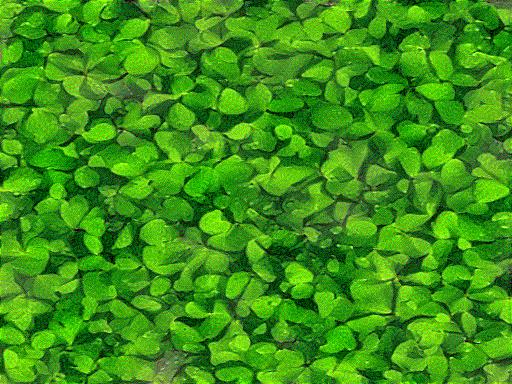}}  \hskip.2em
    \subfloat[]{\includegraphics[height=0.18\linewidth]{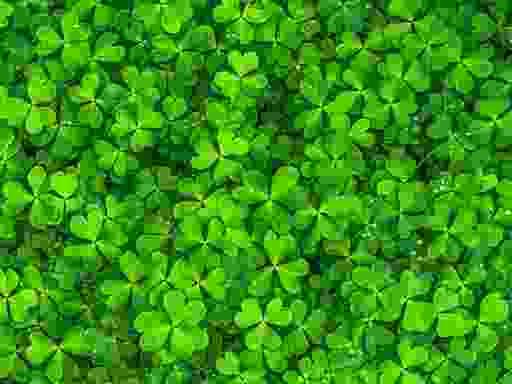}}  \hskip.2em
    \subfloat[]{\includegraphics[height=0.18\linewidth]{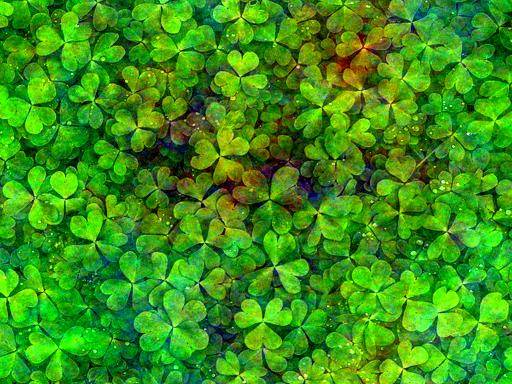}} \\ \vspace{-0.25cm}
    \subfloat[]{\includegraphics[height=0.18\linewidth]{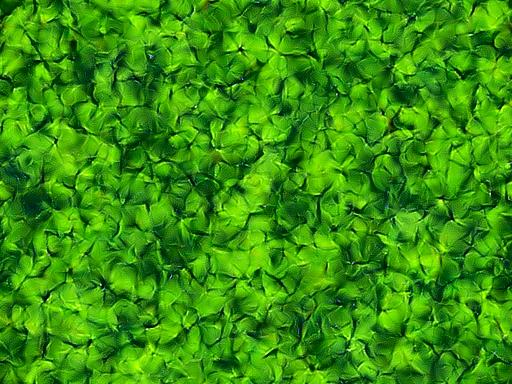}}  \hskip.2em
    \subfloat[]{\includegraphics[height=0.18\linewidth]{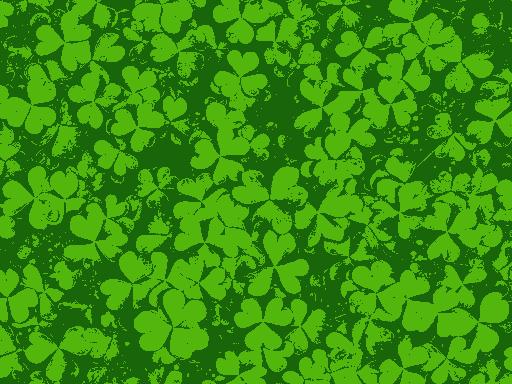}}  \hskip.2em
    \subfloat[]{\includegraphics[height=0.18\linewidth]{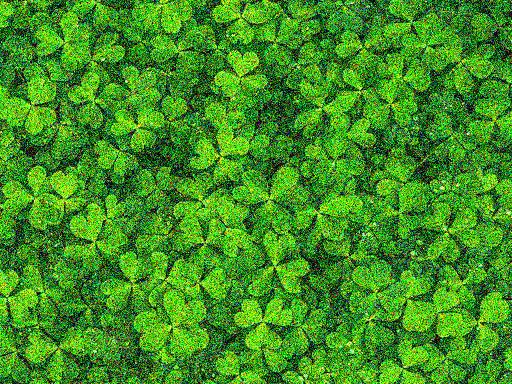}}  \hskip.2em
    \subfloat[]{\includegraphics[height=0.18\linewidth]{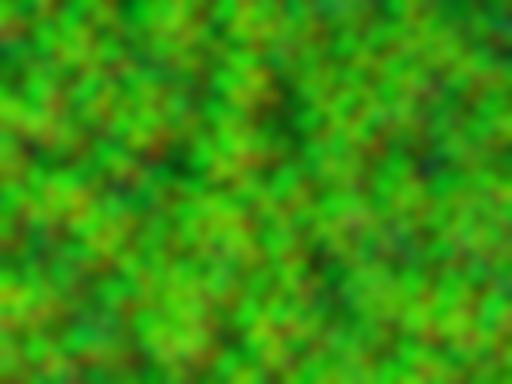}} \\
  \caption{One set of texture images from TQD, ordered according to their rankings by DISTS. \textbf{(a)} Reference image. \textbf{(b)-(p)} Corrupted images ranked by DISTS from high quality to low quality, respectively.}
  \label{fig:rank}
\end{figure*}

\begin{figure}
  \centering
  \includegraphics[height=0.75\linewidth]{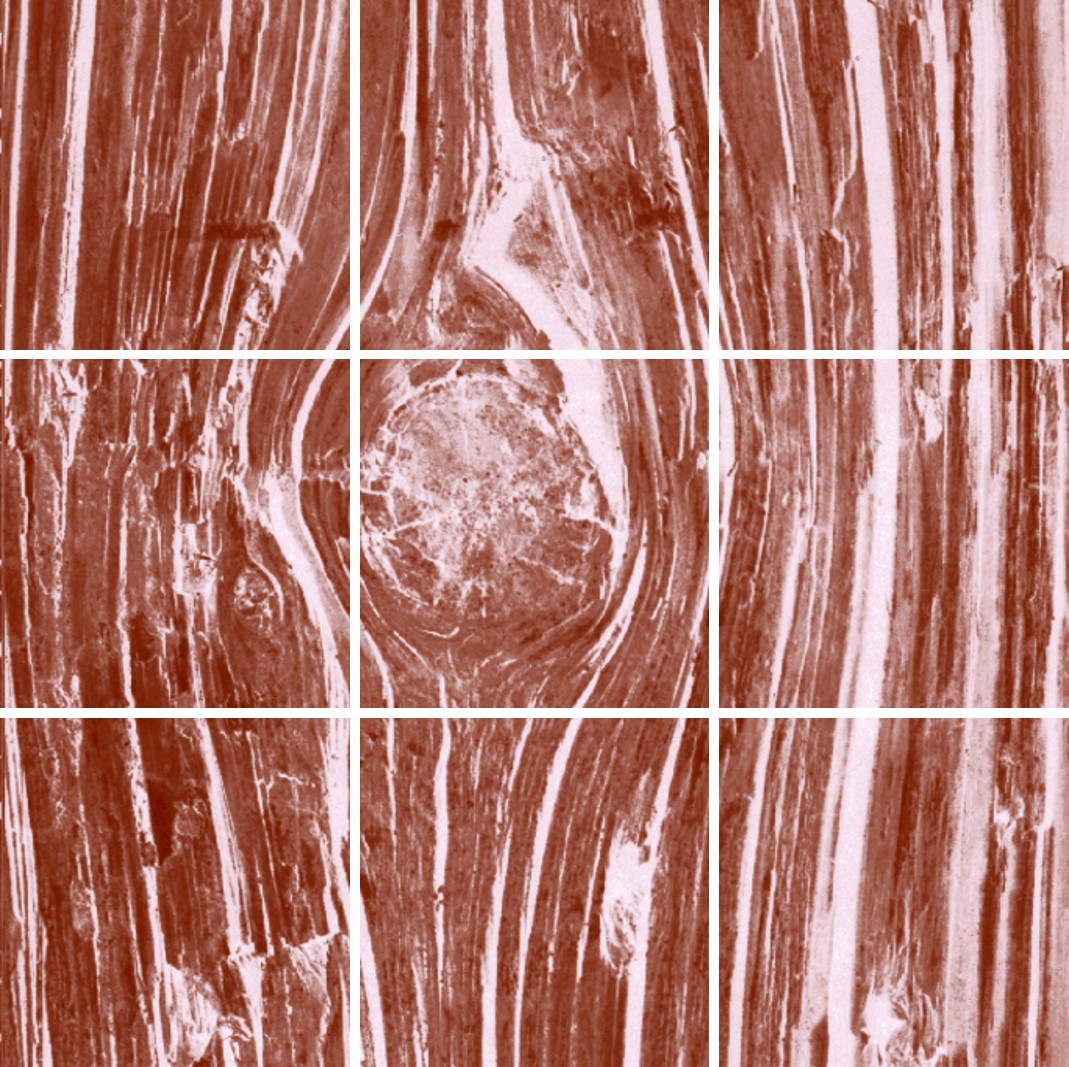}
  \caption{Nine non-overlapping patches sampled from an example texture photograph in the Brodatz color texture dataset.}
  \label{fig:brodat}
\end{figure}

\subsection{Applications to Texture Classification and Retrieval}

We also applied DISTS to texture classification and retrieval. 
We used the grayscale and color Brodatz texture databases~\cite{abdelmounaime2013new} (denoted by GBT and CBT, respectively), each of which contains 
$112$ different texture images. We resampled nine non-overlapping $256\times256\times3$ patches from each texture pattern. Fig.~\ref{fig:brodat} shows a representative texture image from CBT, partitioned into nine patches.

The texture classification problem consists of assigning an unknown sample image to one of the known texture classes.  For each texture, we randomly chose five patches for training, two for validation, and the remaining two for testing. A simple $k$-nearest neighbors ($k$-NN) classification algorithm was implemented, which allowed us to incorporate and compare different similarity models as distance measures. 
The predicted label of a test image was determined  by a majority vote over its $k$ nearest neighbors in the training set, where the value of $k$ was chosen using the validation set.
We implemented a baseline model - the bag-of-words of SIFT features~\cite{lowe2004sift} with $k$-NN.
The classification accuracy results are listed in Table~\ref{tab:acc_brodatz}, where we can see that this baseline model beats most image similarity-based $k$-NN classifiers, except LPIPS (on CBT) and DISTS. This shows that our model is effective at discriminating and classifying textures that are visually different to the human eye. 

The content-based texture retrieval problem consists of searching for images from a large database that are visually similar.
In our experiment, for each texture, we set three patches as the queries, and aimed to retrieve the remaining six patches. Specifically, the distances between each query and the remaining images in the dataset were computed and ranked so as to retrieve the images with minimal distances. To evaluate the retrieval performance, we used 
mean average precision (mAP), which is defined by
\begin{align}
    \mathrm{mAP}=\frac{1}{Q}\sum_{q=1}^Q \left(\frac{1}{K}\sum_{k=1}^{K} P(k)\times \mathrm{rel}(k)\right),
\end{align}
where $Q$ is the number of queries, $K$ is the   number of similar images in the database, $P(k)$ is the precision at cut-off $k$ in the ranked list, and $\mathrm{rel}(k)$ is an indicator function equal to one if the item at rank $k$ is a similar image and zero otherwise.
As seen in Table~\ref{tab:acc_brodatz},  DISTS achieves the best performance on both CBT and GBT datasets. The classification/retrieval errors are primarily due to textures with noticeable inhomogeneities (\eg,  middle patch in Fig.~\ref{fig:brodat}). In addition, the performance on GBT is slightly reduced compared with that on CBT, indicating the importance of color information in these tasks.

Classification and retrieval of texture patches resampled from the same images are relatively easy tasks. We also tested DISTS on a more challenging large-scale texture database, the Amsterdam Library of Textures (ALOT)~\cite{burghouts2009material}, containing photographs of $250$ textured surfaces, from $100$ different viewing angles and illumination conditions. 
Again, we adopted a na\"{i}ve $k$-NN method ($k=100$) using our model as the measure of distance,
and tested it on $20\%$ of the samples randomly selected from the database. Without training on ALOT, DISTS achieves a reasonable classification accuracy of $0.926$, albeit lower than the value of $0.959$ achieved by a knowledge-driven method~\cite{sulc2014fast} with hand-crafted features and support vector machines, and the value of $0.993$ achieved by a data-driven CNN-based method~\cite{cimpoi2016deep}. The primary cause of errors when using DISTS in this task is that images from the same textured surface can appear quite different under different lighting or viewpoint conditions, as seen in the example in Fig.~\ref{fig:soil}. DISTS, which is designed to capture visual appearance only, could likely be improved for this task by fine-tuning the perceptual weights (along with the VGG network parameters) on a small subset of human-labelled ALOT images.

\begin{table}
  \centering
  \caption{Classification and retrieval performance comparison on the Brodatz texture dataset~\cite{abdelmounaime2013new}}
    \begin{tabular}{lccccc}
    \toprule
      \multirow{2}{*}[-3pt]{Method} & \multicolumn{2}{c}{Classification acc.} & \multicolumn{2}{c}{Retrieval mAP} \\ \cmidrule(lr){2-3} \cmidrule(lr){4-5}
      & CBT & GBT & CBT & GBT \\ \hline 
      SSIM~\cite{wang2004image} & 0.397 & 0.210 & 0.371 & 0.145 \\
      CW-SSIM~\cite{wang2005translation} & - & 0.424  & - &  0.351\\
      DeepIQA~\cite{bosse2018deep} &  0.388 & 0.308 & 0.389 & 0.293\\
      PieAPP~\cite{prashnani2018pieapp} & 0.173 & 0.115 & 0.257 & 0.153\\
      LPIPS~\cite{zhang2018unreasonable} & \textbf{0.960} & 0.861 & \textbf{0.951} & 0.839\\
      STSIM~\cite{zujovic2013structural} & - & 0.708  & - & 0.632 \\
      NPTSM~\cite{alfarraj2016content} & - & 0.895  & - & 0.837 \\
      IGSTQA~\cite{golestaneh2018synthesized} & - & 0.862  & - & 0.798 \\
      SIFT~\cite{lowe2004sift} & 0.924 & \textbf{0.928} & 0.859 & \textbf{0.865}\\
      DISTS (ours) & \textbf{0.995} & \textbf{0.968} &\textbf{0.988} & \textbf{0.951} &\\
    \bottomrule
    \end{tabular}
  \label{tab:acc_brodatz} 
\end{table}

\begin{figure*}
  \centering
    \subfloat[Reference]{\includegraphics[height=0.127\linewidth]{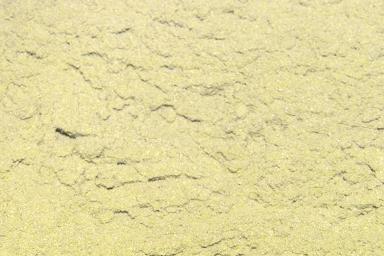}}\hskip.3em
    \subfloat[$D=0.173$]{\includegraphics[height=0.127\linewidth]{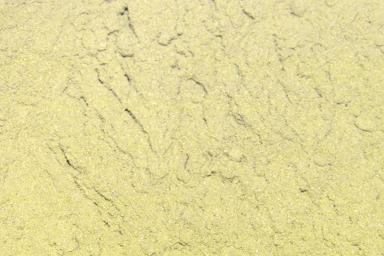}}\hskip.3em
    \subfloat[$D=0.255$]{\includegraphics[height=0.127\linewidth]{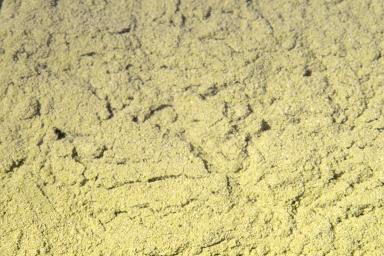}}\hskip.3em
    \subfloat[$D=0.398$]{\includegraphics[height=0.127\linewidth]{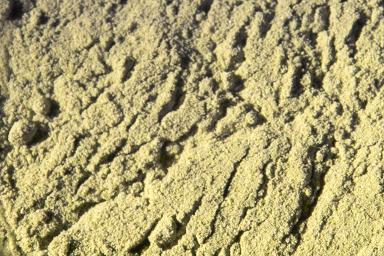}}\hskip.3em
    \subfloat[$D=0.427$]{\includegraphics[height=0.127\linewidth]{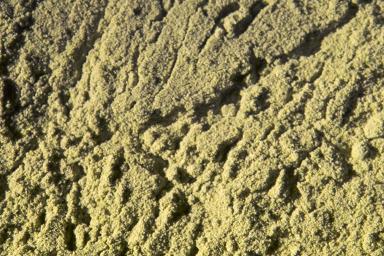}}
  \caption{Five images of ``soil'', photographed under different lighting and viewpoint conditions, from the ALOT dataset. We computed the DISTS score for each of the images (b)-(e) with respect to the reference (a). Consistent with the significantly higher values, (d) and (e) are visually distinct from (a), although all of these images are drawn from the same category. }
  \label{fig:soil}
\end{figure*}

\subsection{Invariance to Geometric Transformations} \label{sec:geo}
Apart from texture similarity, most full-reference IQA measures  fail dramatically when the original and distorted images are misregistered, either globally or locally. The underlying reason is again reliance on the assumption of pixel alignment. Although pre-registration can alleviate this issue, it comes with substantial computational complexity, and does not work well in the presence of severe distortions~\cite{ma2018geometric}. In this subsection,
we investigated the degree of invariance of DISTS to geometric transformations that are imperceptible to the visual system. 

As there are no subject-rated IQA databases designed for this specific purpose, we augmented the LIVE database~\cite{LIVE} (LIVE$\_$Aug) with geometric transformations. In real-world scenarios, an image should first undergo geometric transformations (\eg, camera movement) and then distortions (\eg, JPEG compression). We followed the suggestion in~\cite{ma2018geometric}, and implemented an equivalent but much simpler approach - directly applying the transformations to the original image. Specifically, we augmented reference images using four geometric transformations: 1) shift by $5\%$ pixels in  horizontal direction, 2) clockwise rotation by a degree of $3^{\circ}$, 3) dilation by a factor of $1.05$, and 4) their combination. This yields a set of $(4+1) \times 779$ reference-distortion pairs in the augmented LIVE database. Since the transformations are modest, the quality scores of distorted images with respect to the modified reference images are assumed to be the same as with respect to the original reference image.

The SRCC results of the augmented LIVE database are shown in Table~\ref{tab:aug_cmp}. We found that data-driven methods based on CNNs significantly outperform traditional ones.
Even so, their performance is often made worse by sensitivity to transformations that arises during downsampling without proper Nyquist band limiting.
Trained on augmented data by geometric transformations, GTI-CNN~\cite{ma2018geometric} achieves desirable invariance at the cost of discarding perceptually important features (see Fig.~\ref{fig:optim}). 
DISTS is seen to perform extremely well across all distortions and exhibit a high degree of robustness to geometric transformations, which we believe arises from  1) replacing max pooling with $\ell_2$ pooling, 2) using global quality measurements, and 3) optimizing for invariance to texture resampling (see also Fig.~\ref{fig:aug_example}).

\begin{table}
\centering
  \caption{SRCC comparison of IQA models to human perception using the LIVE database augmented with geometric transformations}
  \setlength{\tabcolsep}{1.6mm}{
    \begin{tabular}{lccccc}
    \toprule
     Method & Translation & Rotation & Dilation & Mixed & Total\\\hline
     PSNR &  0.159 & 0.153 & 0.152 & 0.146 & 0.195\\
     SSIM~\cite{wang2004image} & 0.171 & 0.168 & 0.177 & 0.166 & 0.190\\
     MS-SSIM~\cite{wang2003multiscale}  & 0.165 & 0.174 & 0.198 & 0.174 & 0.177\\
     CW-SSIM~\cite{wang2005translation}  & 0.207 & 0.312 & 0.364 & 0.219 & 0.194\\
     VSI~\cite{zhang2014vsi} & 0.282 & 0.360 & 0.372 & 0.297& 0.309\\
     MAD~\cite{larson:011006} & 0.354 & 0.630 & 0.587 & 0.453& 0.327\\
     VIF~\cite{sheikh2006image}  & 0.296 & 0.433 & 0.522 & 0.387 & 0.294\\    
     FSIM$_\mathrm{c}$~\cite{zhang2011fsim}  & 0.380 & 0.396 & 0.408 & 0.365 & 0.339\\
     NLPD~\cite{laparra2016perceptual} & 0.062 & 0.074 & 0.083 & 0.066 & 0.112\\
     GMSD~\cite{xue2014gradient}  & 0.252 & 0.299 & 0.303 & 0.247 & 0.288\\ \hline
     DeepIQA~\cite{bosse2018deep} & 0.822 & \textbf{0.919} & \textbf{0.918} & 0.881  & 0.859\\
     PieAPP~\cite{prashnani2018pieapp} & 0.850 & 0.903 & 0.902 & 0.879 & 0.874\\
     LPIPS~\cite{zhang2018unreasonable}& 0.811 & 0.908 & 0.893 & 0.861 & 0.779\\
     GTI-CNN~\cite{ma2018geometric} & \textbf{0.864} & 0.906 & 0.904 & \textbf{0.890} & \textbf{0.875}\\ \hline
     DISTS (ours) & \textbf{0.948} & \textbf{0.939} & \textbf{0.946} & \textbf{0.937} & \textbf{0.928}  \\
    \bottomrule
    \end{tabular}}
  \label{tab:aug_cmp}
\end{table}

\begin{figure*}
  \centering
  \begin{tabular}{ccccc}
    \multirow{2}{*}[30pt]{\subfloat[SSIM$\uparrow$ / DISTS$\downarrow$]{\includegraphics[height=0.128\linewidth]{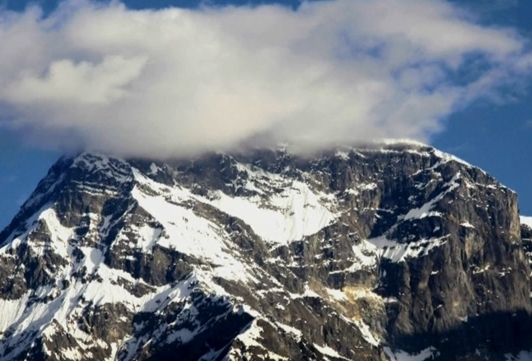}}} & 
    \subfloat[$0.486$ / $0.057$]{\includegraphics[height=0.128\linewidth]{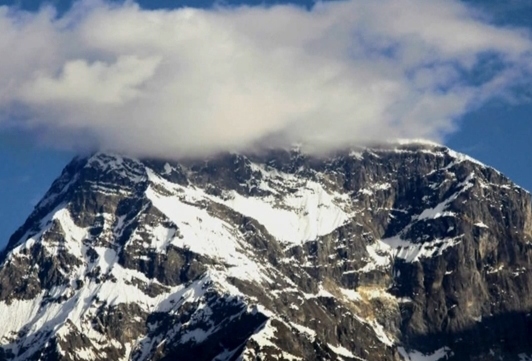}}  \hskip.2em
    \subfloat[$0.482$ / $0.063$]{\includegraphics[height=0.128\linewidth]{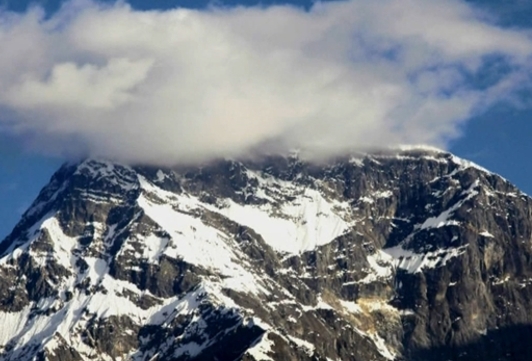}}  \hskip.2em
    \subfloat[$0.493$ / $0.064$]{\includegraphics[height=0.128\linewidth]{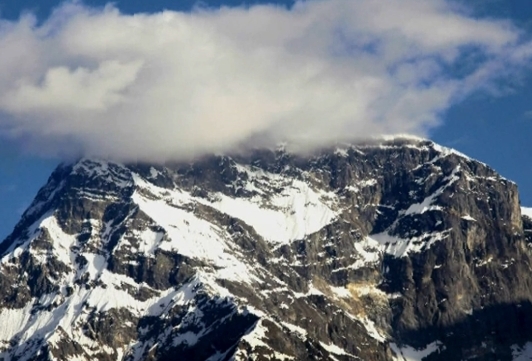}}  \hskip.2em
    \subfloat[$0.630$ / $0.069$]{\includegraphics[height=0.128\linewidth]{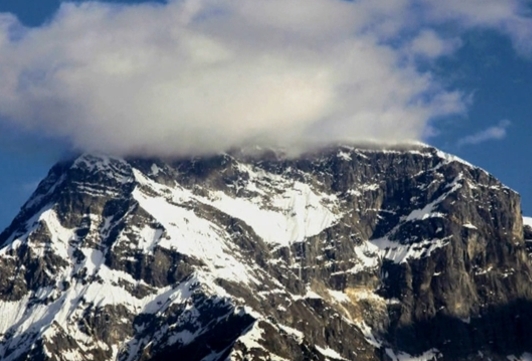}} \\
    &
    \subfloat[$0.539$ / $0.161$]{\includegraphics[height=0.128\linewidth]{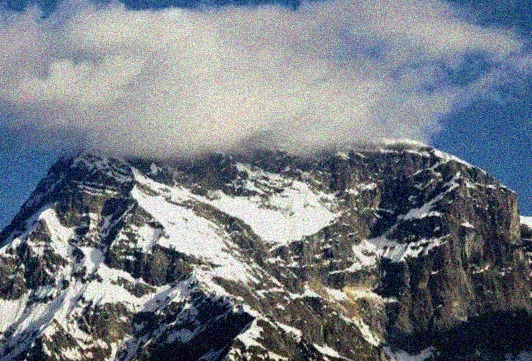}}  \hskip.2em
    \subfloat[$0.637 / 0.329$]{\includegraphics[height=0.128\linewidth]{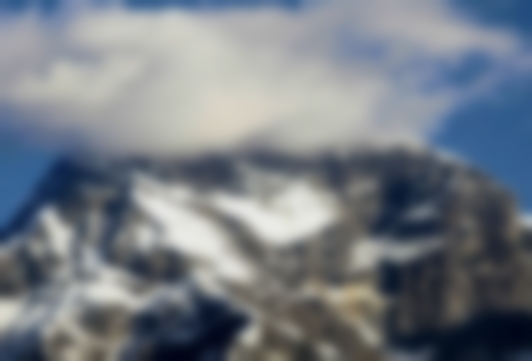}}  \hskip.2em
    \subfloat[$0.705$ / $0.270$]{\includegraphics[height=0.128\linewidth]{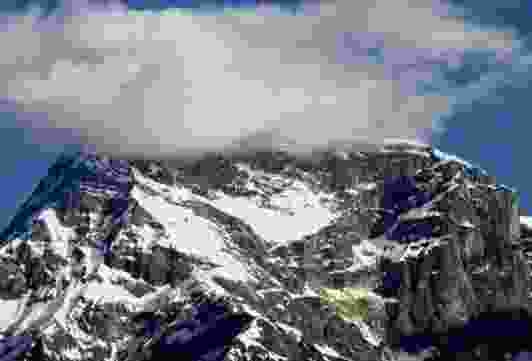}}  \hskip.2em
    \subfloat[$0.730$ / $0.284$]{\includegraphics[height=0.128\linewidth]{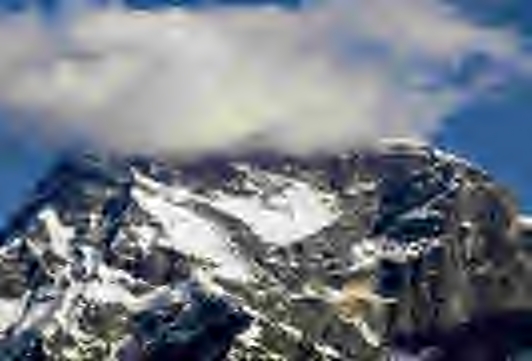}} 
    \end{tabular}
  \caption{A visual example to demonstrate robustness of DISTS to geometric transformations. \textbf{(a)} Reference image. \textbf{(b)} Translated rightward by $5\%$ pixels. \textbf{(c)} Dilated by a factor $1.05$. \textbf{(d)} Rotated by $3$ degrees. \textbf{(e)} Cloud movement. \textbf{(f)} Corrupted with additive Gaussian noise. \textbf{(g)} Gaussian blur. \textbf{(h)} JPEG compression. \textbf{(i)} JPEG2000 compression. Below each image are the values of SSIM and DISTS, respectively. SSIM values are similar or better (larger) for the bottom row, whereas our model reports better (smaller) values for the top row, consistent with human perception.} 
  \label{fig:aug_example}
\end{figure*}

\subsection{Ablation Study} \label{sec:ablation}
\added{In this subsection, we conducted ablation experiments to single out the individual contributions of key modifications of DISTS, in comparison to the most closely related alternative - LPIPS. We trained a series of intermediate models between LPIPS and DISTS:}
\begin{itemize}
    \item [(a)] Original LPIPS; 
    \item [(b)] Replace max pooling in LPIPS with $\ell_2$ pooling;
    \item [(c)] Add the input image on the top of (b);
    \item [(d)] Replace the Euclidean distance in LPIPS with local SSIM measurements (within a sliding window of size $11\times11$) on top of (c);
    \item [(e)] Replace the Euclidean distance in LPIPS with global SSIM measurements on top of (c);
    \item [(f)] Train (c) by adding  the $E_2$ term in Eq. (\ref{eq:loss2});
    \item [(g)] Train (d) by adding the $E_2$ term;
    \item [(h)] Train (e) by adding the $E_2$ term, which is equivalent to DISTS.
\end{itemize}

Performance of these models is shown in Table \ref{tab:ablation}, from which we draw several conclusions. First, $\ell_2$ pooling is slightly better than max pooling. The main motivation of adopting $\ell_2$ pooling is to de-alias the intermediate representations, as documented in \cite{henaff2015geodesics}. Second, incorporating the input image in the representation has little impact on the performance, but it ensures a unique minimum of DISTS, which is beneficial in perceptual optimization~\cite{ding2020optim}. Third, the global SSIM-like distance outperforms the Euclidean distance, especially in measuring similarity of visual textures and invariance to geometric transformations. We  also tested local SSIM measurements within a sliding window size of $11\times11$ (d), which gives inferior performance.
Last, training with the $E_2$ term is important for texture-related tasks, improving invariance to geometric transformations, although it slightly hurts the performance on standard IQA databases. We concluded that the improved quality prediction and texture similarity performance of DISTS relative to LPIPS is due to the combination of these key modifications.


\begin{table*}
\small
  \centering  
    \caption{\added{Ablation experiments: proposed DISTS model (last line) compared to LPIPS (first line), and intermediate variations. All models trained on KADID}}
    \setlength{\tabcolsep}{0.9mm}{
    \begin{tabular}{l cccccc cccccc ccc}
    \toprule
    \multirow{3}{*}[-3pt]{Model} & \multicolumn{6}{c}{Quality prediction}&\multicolumn{6}{c}{Texture similarity}&\multicolumn{3}{c}{Geometric invariance} \\
    \cmidrule(lr){2-7} \cmidrule(lr){8-13} \cmidrule(lr){14-16}
     &  \multicolumn{3}{c}{LIVE~\cite{LIVE}} & \multicolumn{3}{c}{TID2013~\cite{Ponomarenko201557}} 
     & \multicolumn{3}{c}{SynTEX~\cite{golestaneh2015effect}} & \multicolumn{3}{c}{TQD (proposed)} &  \multicolumn{3}{c}{LIVE\_Aug} \\
     \cmidrule(lr){2-4} \cmidrule(lr){5-7} \cmidrule(lr){8-10}
     \cmidrule(lr){11-13} \cmidrule(lr){14-16}
     & PLCC & SRCC & KRCC & PLCC & SRCC & KRCC & PLCC & SRCC & KRCC & PLCC & SRCC & KRCC & PLCC & SRCC & KRCC \\ \hline
    (a) LPIPS               & 0.934 & 0.936 & 0.769 & 0.850 & 0.824 & 0.626 & 0.591 & 0.589 & 0.452 & 0.403 & 0.401 & 0.302 & 0.801 & 0.793 & 0.629\\
    (b) a + $\ell_2$ pooling& 0.937 & 0.938 & 0.770 & 0.851 & 0.824 & 0.626 & 0.594 & 0.592 & 0.459 & 0.410 & 0.406 & 0.305 & 0.807 & 0.802 & 0.633\\
    (c) b + input image     & 0.935 & 0.935 & 0.768 & 0.851 & 0.825 & 0.627 & 0.582 & 0.581 & 0.449 & 0.410 & 0.409 & 0.303 & 0.795 & 0.789& 0.625\\
    (d) c + local SSIM      & 0.950 & 0.951 & 0.797 & 0.853 & 0.828 & 0.631 & 0.738 & 0.744 & 0.602 & 0.664 & 0.667 & 0.559 & 0.798 & 0.790 & 0.626\\
    (e) c + global SSIM     & \textbf{0.955} & \textbf{0.957} & \textbf{0.816 }& \textbf{0.859} & \textbf{0.835} & \textbf{0.641} & \textbf{0.868} & \textbf{0.877} & \textbf{0.739} & \textbf{0.780} & \textbf{0.795} & \textbf{0.698} & \textbf{0.899} & \textbf{0.881} & \textbf{0.724}\\
    (f) c + $E_2$ term      & 0.934 & 0.935 & 0.768 & 0.791 & 0.776 & 0.608 & 0.780 & 0.782 & 0.630 & 0.680 & 0.685 & 0.588 & 0.830 & 0.823 & 0.655\\
    (g) d + $E_2$ term      & 0.929 & 0.931 & 0.766 & 0.801 & 0.783 & 0.615 & 0.774 & 0.778 & 0.625 & 0.672 & 0.678 & 0.579 & 0.820 & 0.816 & 0.649\\
    (h) e + $E_2$ = DISTS     & \textbf{0.954} & \textbf{0.954} & \textbf{0.811} & \textbf{0.855} & \textbf{0.830} & \textbf{0.639} &\textbf{ 0.901} & \textbf{0.923}          & \textbf{0.759} & \textbf{0.903} & \textbf{0.910} & \textbf{0.785} & \textbf{0.931} & \textbf{0.928} & \textbf{0.762}\\
    \bottomrule
  \end{tabular}}
  \label{tab:ablation} 
\end{table*}

\section{Conclusions}
We have presented a new full-reference IQA method, DISTS, which is the first of its kind with built-in tolerance to texture resampling. 
Our model unifies structure and texture similarity, providing good predictions of human quality ratings on both textures and natural photographs, is robust to mild geometric distortions, and performs well in texture classification and retrieval.

DISTS is based on the pre-trained VGG network for object recognition. By computing the global means of convolution responses at each stage, we established a universal parametric texture model similar to that of Portilla \& Simoncelli~\cite{portilla2000parametric}. These statistical measurements provide a rich but relatively low-dimensional characterization of texture appearance, as verified using synthesis (Fig. \ref{eq:texture_syn}).
Despite the empirical success, we believe an important direction for future work is to analyze this ``black box'' to understand 1) what and how certain texture features and attributes are captured by the pre-trained network, and 2) the importance of cascaded convolution and subsampled pooling in summarizing useful texture information. It is also of interest to extend the current model to measure distortions locally, as is done in SSIM. In this case, the distance measure could be reformulated to adaptively select between structure and texture measures as appropriate, instead of linearly combining them with fixed weights.

The most direct use of IQA measures is for performance assessment and comparison of image processing systems. But perhaps more importantly, they may be used to optimize image processing methods, so as to improve the visual quality of their results. In this context, most existing IQA measures present major obstacles due to the fact that they lack desired mathematical properties that aid optimization (\eg, injectivity, differentiability and convexity).
In many cases, they rely on surjective mappings, and minima are non-unique (see Fig.~\ref{fig:optim}). Although DISTS enjoys several advantageous mathematical properties, it is still highly non-convex (with abundant saddle points and plateaus), and recovery from random noise using stochastic gradient descent methods (see Fig.~\ref{fig:optim}) requires many more iterations than for SSIM. 
In practice, the larger the weight of the structure term $s$ at the zeroth stage ($\beta_{0j}$ in Eq.~(\ref{eq:s2})), the faster the optimization converges. 
However, to reach a reasonable level of texture invariance,  the learned $\sum_{i,j}\alpha_{ij}$ should be larger than $\sum_{i,j}\beta_{ij}$, \comment{reflecting that global mean terms $l$ defined in Eq. (\ref{eq:s1}) play a dominant role in our metric} hindering optimization. 
We are currently analyzing DISTS in the context of perceptual optimization. Our initial results indicate that DISTS-based optimization of image processing applications, including 
denoising, deblurring, super-resolution, and compression, can lead to noticeable improvements in visual quality~\cite{ding2020optim}.

\ifCLASSOPTIONcaptionsoff
  \newpage
\fi



%

\bibliographystyle{IEEEtran}
\end{document}